\documentclass{article}

\usepackage{microtype}
\usepackage{graphicx}
\usepackage{subfigure}
\usepackage{booktabs} 

\microtypecontext{spacing=nonfrench}
\usepackage{booktabs}
\usepackage{amsmath,amsfonts,pifont}
\usepackage{algorithmic}
\usepackage{multirow}
\usepackage{tabu}
\usepackage{tabularx}
\usepackage{textcomp}
\usepackage{xcolor}
\usepackage{tabularray}
\UseTblrLibrary{diagbox}
\usepackage{booktabs} 
\usepackage{multirow, makecell}
\PassOptionsToPackage{table,xcdraw,dvipsnames}{xcolor}
\usepackage{colortbl}
\usepackage{soul}
\usepackage{bm}
\usepackage{subcaption}
\usepackage{xcolor}
\usepackage{amsthm}
\usepackage{stfloats}

\newtheorem{claim}{Claim}

\newcolumntype{?}{!{\vrule width 1pt}}
\newcommand{\commenta}[1]

\newcommand{\para}[1]{\vspace{1pt} \noindent\textbf{#1} \hspace{1pt}}
\newcommand{\etal}{\mbox{\it{et al.}}}
\newcommand{\ie}{\mbox{\it{i.e.,\ }}}
\newcommand{\eg}{\mbox{\it{e.g.,\ }}}
\DeclareMathOperator*{\argmax}{arg\,max}

\newcommand{\s}{\mathit{s}}
\newcommand{\ia}{\mathit{a}}

\def\Snospace~{Section }

\usepackage{hyperref}

\usepackage[accepted]{icml2024}

\usepackage{amssymb}
\usepackage{mathtools}
\usepackage{amsthm}
\usepackage{xspace}

\usepackage[capitalize,noabbrev]{cleveref}

\theoremstyle{plain}
\newtheorem{theorem}{Theorem}[section]
\newtheorem{proposition}[theorem]{Proposition}
\newtheorem{lemma}[theorem]{Lemma}

\theoremstyle{definition}

\newtheorem{assumption}[theorem]{Assumption}
\theoremstyle{remark}

\newcommand{\sref}[2]{\hyperref[#2]{#1 \ref{#2}}}

\usepackage[textsize=tiny]{todonotes}

\begin{document}

\twocolumn[
\icmltitle{RICE: Breaking Through the Training Bottlenecks of Reinforcement Learning with Explanation}



\icmlsetsymbol{equal}{*}

\begin{icmlauthorlist}
\icmlauthor{Zelei Cheng}{equal,nu}
\icmlauthor{Xian Wu}{equal,nu}
\icmlauthor{Jiahao Yu}{nu}
\icmlauthor{Sabrina Yang}{sch}
\icmlauthor{Gang Wang}{uiuc}
\icmlauthor{Xinyu Xing}{nu}
\end{icmlauthorlist}

\icmlaffiliation{nu}{Department of Computer Science, Northwestern University, Evanston, Illinois, USA}
\icmlaffiliation{sch}{Presentation High School, San Jose, California, USA}
\icmlaffiliation{uiuc}{Department of Computer Science,  University of Illinois at Urbana-Champaign, Urbana, Illinois, USA}

\icmlcorrespondingauthor{Xinyu Xing}{xinyu.xing@northwestern.edu}

\icmlkeywords{Machine Learning, ICML}

\vskip 0.3in
]



\printAffiliationsAndNotice{\icmlEqualContribution} 
\begin{abstract}
Deep reinforcement learning (DRL) is playing an increasingly important role in real-world applications. However, obtaining an optimally performing DRL agent for complex tasks, especially with sparse rewards, remains a significant challenge. The training of a DRL agent can be often trapped in a bottleneck without further progress. In this paper, we propose {\tt RICE}, an innovative refining scheme for reinforcement learning that incorporates explanation methods to break through the training bottlenecks. The high-level idea of {\tt RICE} is to construct a new initial state distribution that combines both the default initial states and critical states identified through explanation methods, thereby encouraging the agent to explore from the mixed initial states. Through careful design, we can theoretically guarantee that our refining scheme has a tighter sub-optimality bound.  We evaluate {\tt RICE} in various popular RL environments and real-world applications. The results demonstrate that {\tt RICE} significantly outperforms existing refining schemes in enhancing agent performance.

\end{abstract}

\section{Introduction}
\label{sec:intro}


Deep reinforcement learning (DRL) has shown promising performance in various applications ranging from playing simulated games~\cite{todorov2012mujoco, mnih2013playing, oh2016control, cai2023imitation} to completing real-world tasks such as navigating autonomous vehicles and performing cybersecurity attacks and defenses~\cite{bar2022werlman, vyas2023automated, anderson2018learning, wang2023autonomous}. 
However, training an optimal DRL agent for complex tasks, particularly in environments with sparse rewards, presents a significant challenge. Often cases, the training of a DRL agent can hit a bottleneck without making further process: its sub-optimal performance becomes evident when it makes common mistakes or falls short of achieving the final goals.


When the DRL agent hits its training bottleneck, a refinement strategy can be considered, especially if the agent is already locally optimal. To refine the locally optimal DRL agent, one method is to analyze its behavior and patch the errors it made. 
A recent work~\cite{cheng2023statemask} proposes StateMask to identify critical states of the agent using an explanation method. One utility of StateMask is patching the agent's error, which fine-tunes the DRL agent starting from the identified critical states (denoted as ``StateMask-R''). 
However, such an approach suffers from two drawbacks. On the one hand, initializing solely from critical states will hurt the diversity of initial states, which can cause overfitting (see \autoref{appendix:malware}). On the other hand, fine-tuning alone cannot help the DRL agent jump out of the local optima. These observations drive us to rethink how to design a proper initial distribution and apply exploration-based techniques to patch previous errors.

Another reason behind the training bottleneck can be the poor choice of the training algorithm. Naturally, to improve performance, the developer needs to select another DRL training algorithm to re-train the DRL agent. However, for complex DRL tasks, re-training the agent {\em from scratch} is too costly. 
For instance, for AlphaStar~\cite{vinyals2019grandmaster} to attain grandmaster-level proficiency in StarCraft, its training period exceeds one month with TPUs.
Retraining an agent of this level can incur a cost amounting to millions of dollars~\cite{agarwal2022reincarnating}. 
Therefore, existing research has investigated the reuse of previous DRL training (as prior knowledge) to facilitate re-training~\cite{ho2016generative, fu2018learning, cai2022seeing}. The most recent example is Jump-Start Reinforcement Learning (JSRL) proposed by \citet{uchendu2023jump} which leverages a pre-trained policy to design a curriculum to guide the training of a self-improving exploration policy. However, their selection of exploration frontiers in the curriculum is random, which cannot guarantee that the exploration frontiers have positive returns. This motivates us to incorporate explanation methods to scrutinize the pre-trained policy and design more effective exploration frontiers.

In this work, we propose {\tt RICE}\footnote{The source code of {\tt RICE} can be found in \url{https://github.com/chengzelei/RICE}}, a {\bf R}efining scheme for Re{\bf I}nfor{\bf C}ement learning with {\bf E}xplanation. We first leverage a state-of-the-art explanation method to derive a step-level explanation for the pre-trained DRL policy. The explanation method identifies the most critical states (\ie steps that contribute the most to the final reward of a trajectory), which will be used to construct the exploration frontiers. Based on the explanation results, we construct a mixed initial state distribution that combines the default initial states and the identified critical states to prevent the overfitting problem. By forcing the agent to revisit these exploration frontiers, we further incentivize the agent to explore starting from the frontiers. Through exploration, the agent is able to expand state coverage, and therefore more effectively break through the bottlenecks of reinforcement learning training. Our theoretical analysis shows that this method achieves a tighter sub-optimality bound by utilizing this mixed initial distribution (see \autoref{subsec:theory}). 

In addition, we introduce key improvements to the state-of-the-art explanation method StateMask~\cite{cheng2023statemask} to better facilitate our refining scheme. We reformulate the objective function and add a new reward bonus for encouraging blinding when training---this significantly simplifies the implementation without sacrificing the theoretical guarantee.

\para{Evaluation and Findings.}
We evaluate the performance of {\tt RICE} using four MuJoCo games and four DRL-based real-world applications,
including cryptocurrency mining~\cite{bar2022werlman}, autonomous cyber defense (Cage Challenge 2)~\cite{cage_challenge_2_announcement}, autonomous driving~\cite{li2022metadrive}, and malware mutation~\cite{raff2017malware}. 
We show that the explanation derived from our new design demonstrates similar fidelity to the state-of-the-art technique StateMask~\cite{cheng2023statemask} with {\em significantly improved training efficiency}. With the explanation results, we show our refining method can produce higher performance improvements for the pre-trained DRL agent, in comparison with existing approaches including JSRL~\cite{uchendu2023jump} and the original refining method from StateMask~\cite{cheng2023statemask}. 


In summary, our paper has the following contributions:
\begin{itemize}
    \item We develop a refining strategy to break through the bottlenecks of reinforcement learning training with an explanation (which is backed up by a theoretical analysis). We show our refining method performs better than those informed by random explanation. 
    \item We propose an alternative design of StateMask to explain the agent's policy in DRL-based applications. Experiments show that our explanation has comparable fidelity with StateMask while improving efficiency.
    \item With extensive evaluations and case studies, we illustrate the benefits of using {\tt RICE} to improve a pre-trained policy. 
\end{itemize}

\section{Related Work}
\subsection{Explanation-based Refining}
\label{sec:literature1}
Recently, there has been some work that leverages the DRL explanation to improve the agent's performance. These explanations can be derived from either human feedback or automated processes.  \citet{guan2021widening, van2022correct} propose to utilize human feedback to correct the agent's failures. More specifically, when the agent fails, humans (can be non-experts) are involved to point out how to avoid such a failure (\ie what action should be done instead, and what action should be forbidden). Based on human feedback, the DRL agent gets refined by taking human-advised action in those important time steps and finally obtains the corrected policy. The downside is that it relies on humans to identify critical steps and craft rules for alternative actions. This can be challenging for a large action space, and the retraining process is ad-hoc and time-consuming. \citet{cheng2023statemask, yu2023airs} propose to use step-level DRL explanation methods to automatically identify critical time steps and refine the agent accordingly. It initiates the refining process by resetting the environment to the critical states and subsequently resumes training the DRL agents from these critical states. Empirically, we observe that this refining strategy can easily lead to overfitting (see \autoref{appendix:malware}). Instead, we propose a novel refining strategy with theoretical guarantees to improve the agent's performance.

\subsection{Leveraging Existing Policy}
The utilization of existing policies to initialize RL and enhance exploration has been explored in previous literature. 
Some studies propose to ``roll-in'' with an existing policy for better exploration, as demonstrated in works \cite{agarwal2020pc, li2023understanding}. Similar to our approach, JSRL \cite{uchendu2023jump} incorporates a guide policy for roll-in, followed by a self-improving exploration policy. Technically, JSRL relies on a curriculum for the gradual update of the exploration frontier. However, the curriculum may not be able to truly reflect the key reasons why the guide policy succeeds or fails. Therefore, we propose to leverage the explanation method to automatically identify crucial states, facilitating the rollout of the policy by integrating these identified states with the default initial states. In \autoref{sec:eval}, we empirically demonstrate that JSRL performs poorly in our selected games. \citet{chang2023learning} propose PPO++ that reset the environment to a mixture of the default initial states and the visited states of a guide policy (\ie a pre-trained policy). It can be viewed as a special case in our framework, \ie constructing a mixed initial distribution with a random explanation. However, we claim that not all visited states of a pre-trained policy are informative and our theoretical analysis and experiments both show that {\tt RICE} based on our explanation method outperforms the refining method based on a random explanation.

\section{Proposed Technique}
\label{sec:bg}

\begin{figure*}[t]
    \centering
    \includegraphics[width=1.0\linewidth]{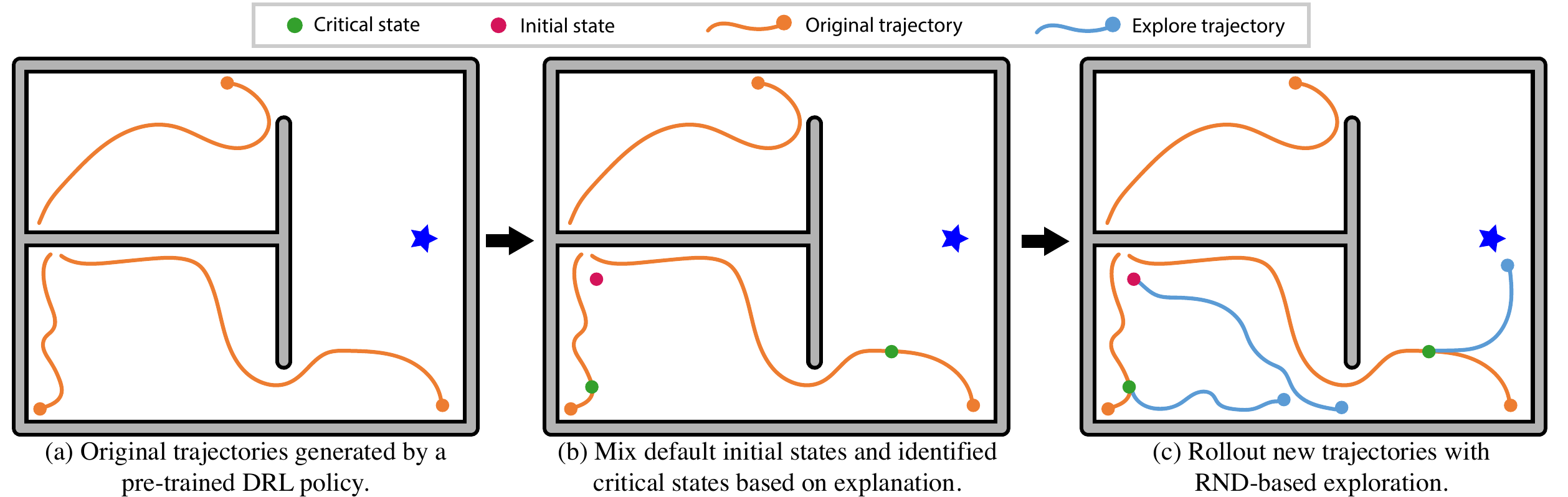}
    \caption{Given a pre-trained DRL policy that is not fully optimal (a), we propose the {\tt RICE} algorithm that resets the RL agent to specific visited states (a mixture of default initial states and identified critical states) (b), followed by an exploration step initiated from these chosen states (c).
    }   
    \label{fig:framework}
    \vspace{-10pt}
\end{figure*}

\subsection{Problem Setup and Assumption}
We model the problem as a Markov Decision Process (MDP), which is defined as a tuple
$\langle \mathcal{S}$, $\mathcal{A}$, $\mathcal{P}$, $\rho$, $\mathcal{R}$, $\gamma \rangle$. In this tuple, $\mathcal{S}$ and $\mathcal{A}$ are the state and action set, where each $\s_{t}$ and $\ia_{t}$ represents the state and action of the agent at time $t$. $P: \mathcal{S} \times  \mathcal{A} \rightarrow \Delta(\mathcal{S})$ is the state transition function,  $\mathcal{R}: \mathcal{S}  \times \mathcal{A}  \rightarrow \mathbb{R}$ is the reward function. $\gamma \in (0, 1)$ is the discount factor. For a policy $\pi(a|s)$: $\mathcal{S} \rightarrow \mathcal{A}$, the value function and $Q$-function is defined as $V^{\pi}(s) = \mathbb{E}_{\pi}\left[ \sum_{t=0}^{\infty} \gamma^t R(s_t, a_t) \mid s_0 = s \right]$ and $Q^{\pi}(s, a) = \mathbb{E}_{\pi}\left[ \sum_{t=0}^{\infty} \gamma^t R(s_t, a_t) \mid s_0 = s, a_0 = a \right]$. The advantage function for the policy $\pi$ is denoted as $A^{\pi}(s, a) = Q^{\pi}(s, a) - V^{\pi}(s)$. We assume the initial state distribution is given by $\rho$: $s_0 \sim \rho$. The goal of RL is to find an optimal policy $\pi^*$ that maximizes its expected total reward : $\pi^* = \argmax_{\pi} \mathbb{E}_{s\sim \rho}\left[V^{\pi}(s)\right]$. Besides, we also introduce the state occupancy distribution and the state-action occupancy measure for $\pi$, denoted as $d_{\rho}^\pi(s)=(1-\gamma) \sum_{t=0}^{\infty} \gamma^t \operatorname{Pr}^{\pi}\left(s_t=s \mid s_0 \sim \rho\right)$ and $d_{\rho}^\pi(s,a)=d_{\rho}^\pi(s) \pi(a|s)$.

In our setting, we have a pre-trained policy denoted as $\pi$, which may be sub-optimal. Our objective is to break through the training bottlenecks of the pre-trained policy with an explanation. Rather than re-training from scratch, we propose to utilize explanation to take full advantage of the guidance of the pre-trained policy $\pi$. Importantly, we do not assume knowledge of the original training algorithm used for policy $\pi$. And we make the following assumptions about the quality of $\pi$.
\begin{assumption}
\label{assumption:random}
 Given a random policy $\pi^{r}$, we have $\mathbb{E}_{a \sim \pi^{r}}[A^{\pi}(s, a)] \leq 0$, $\forall s$.
 \vspace{-3pt}
\end{assumption}

Intuitively, the above assumption implies that taking an action based on a random policy $\pi^r$ will provide a lower advantage than taking actions based on the policy $\pi$. This is a reasonable assumption since $\pi$ is a pre-trained policy, thus it would perform much better than an untrained (\ie random) policy. 

\begin{assumption}
\label{assumption:mismatch}
The pre-trained policy $\pi$ cover the states visited by the optimal policy $\pi^{*}$: $\left\|\frac{d_\rho^{\pi^*}}{d_\rho^{\pi}}\right\|_{\infty} \le C$, where $C$ is a constant.
 \vspace{-3pt}
\end{assumption}

In other words, \sref{Assumption}{assumption:mismatch} requires that the pre-trained policy visits all good states in the state space. Note that it is a standard assumption in the online policy gradient learning ~\cite{agarwal2021theory, uchendu2023jump, li2023understanding} and is much weaker than the \textit{single policy concentrateability coefficient} assumption~\cite{rashidinejad2021bridging, xie2021policy}, which requires the pre-trained policy visits all good state-action pairs. The ratio in \sref{Assumption}{assumption:mismatch} is also referred to as the distribution mismatch coefficient.

\subsection{Technical Overview}
Recall our goal is to refine the pre-trained DRL agent to break through the training bottlenecks. At a high level, the {\tt RICE} algorithm integrates a roll-in step, where the RL agent is reset to specific visited states, followed by an exploration step initiated from these chosen states. During the roll-in step, we draw inspiration from established RL-explanation methods ~\cite{ puri2019explain, guo2021edge, cheng2023statemask} to identify critical states, referred to as exploration frontiers, within the given policy $\pi$. As depicted in \autoref{fig:framework}, when presented with a trajectory sampled from the policy $\pi$, we employ a step-level explanation method -- StateMask ~\cite{cheng2023statemask} to identify the most crucial time steps influencing the final rewards in this trajectory. Subsequently, we guide the RL agent to revisit these selected states. The rationale behind revisiting these states lies in their ability to offer an expanded initial state distribution compared to $\rho$, thereby enabling the agent to explore diverse and relevant states it might otherwise neglect. Additionally, we introduce a mixing of these selected states with the initial states sampled from $\rho$. This mixing approach serves the purpose of preventing the agent from overfitting to specific states. In \autoref{subsec:theory}, we theoretically show that {\tt RICE} achieves a tighter regret bound through the utilization of this mixed initial distribution.

Then, we propose an exploration-based method to further enhance the DRL agent's performance. The high-level idea is to incentivize the agent to explore when initiating actions from these frontiers. Intuitively, the pre-trained policy $\pi$ might converge to a local optimal, as shown in \autoref{fig:framework}. Through exploration, we aim to expand state coverage by rewarding the agent for visiting novel states, thereby increasing the likelihood of successfully completing the task. Specifically, we utilize the Proximal Policy Optimization (PPO) algorithm ~\cite{schulman2017proximal} for refining the DRL agent, leveraging the monotonicity of PPO.

\subsection{Technique Detail}
\label{sec:tech}

\para{Step-level Explanation.}
We leverage a state-of-the-art explanation method StateMask~\cite{cheng2023statemask}. At a high level, StateMask parameterizes the importance of the target agent's current time step as a neural network model (\ie mask network). This neural network takes the current state as input and then outputs this step's importance score with respect to the agent's final reward. To do so, StateMask learns a policy to “blind” the target agent at certain steps without changing the agent's final reward. Specifically, for an input state $s_{t}$, the mask net outputs a binary action $a_{t}^{m}$ of either ``zero'' or ``one'', and the target agent will sample the action $a_{t}$ from its policy. 
The final action is determined by the following equation
\vspace{-5pt}
\begin{equation}
a_t \odot a_t^m = \left \{
\begin{aligned}
&a_t, & \text{if } a_t^m=0 \, ,\\
&a_{\text{random}} & \text{if } a_t^m=1 \, ,
\end{aligned}
\right.    
\label{mask_operator}
 \vspace{-5pt}
\end{equation}
The mask net is then trained to minimize the following objective function:
\begin{equation}
    \begin{aligned}
    J(\theta) = \text{min } |\eta(\pi) - \eta(\bar{\pi})| \, ,
    \end{aligned}
    \label{eq:objective}
\vspace{-5pt}
\end{equation}
where $\pi$ denotes the policy of the target agent (\ie our pre-trained policy), $\bar{\pi}$ denotes the policy of the perturbed agent (\ie integrating the random policy and the target agent $\pi$ via the mask network $\tilde{\pi}$), $\eta(\cdot)$ is the expected total reward of an agent by following a certain policy. To solve the Eqn.~\eqref{eq:objective} with monotonicaly guarantee, StateMask carefully designs a surrogate function and utilize the prime-dual methods to optimize the $\tilde{\pi}$. However, we can optimize the learning process of mask net within our setting to enhance simplicity. Specifically, we have the following theorem
\begin{theorem}
Under \sref{Assumption}{assumption:random}, we have $\eta(\bar{\pi})$ upper-bounded by $\eta(\pi)$: $\eta(\bar{\pi}) \leq \eta(\pi)$.
\label{theorem:upper_bound}
\end{theorem}
The proof of the theorem can be found in \sref{Appendix}{appendix:proof_theorem33}. Leveraging this theorem, we can transform the objective function to $J(\theta) = \text{max } \eta(\bar{\pi})$. With this reformulation, we can utilize the vanilla PPO algorithm to train the state mask without sacrificing the theoretical guarantee. However, na\"ively maximizing the expected total reward may introduce a trivial solution to the problem which is to {\em not blind the target agent at all} (always outputs ``0''). To tackle this problem, we add an additional reward by giving an extra bonus when the mask net outputs ``1''. The new reward can be written as $R'(s_t, a_t) = R(s_t, a_t) + \alpha a_t^m$ where $\alpha$ is a hyper-parameter. We present the learning process of the mask network in \autoref{alg:masknet}. By applying this resolved mask to each state, we will be able to assess the state importance (\ie the probability of mask network outputting ``0'') at any time step.

\begin{algorithm}[t]
\caption{Training the Mask Network.}
\label{alg:masknet}
\begin{algorithmic}
\footnotesize
\STATE {\bfseries Input:} Target agent's policy $\pi$
\STATE {\bfseries Output:} Mask network $\tilde{\pi}_{\theta}$ 
\STATE {\bfseries Initialization:} Initialize the weights $\theta$ for the mask net $\tilde{\pi}_{\theta}$
\STATE $\theta_{old} \leftarrow \theta$
\FOR{iteration=1, 2, \dots}
\STATE Set the initial state $s_0 \sim \rho$
\STATE$\mathcal{D} \leftarrow \emptyset$ 
\FOR{t=0 to T}
\STATE Sample $a_t \sim \pi(a_t | s_t)$
\STATE Sample $a_t^m \sim \tilde{\pi}_{\theta_{old}}(a_t^m | s_t)$
\STATE Compute the actual taken action $a \leftarrow a_t \odot a_t^m$ 
\STATE $(s_{t+1}, R_t^{\prime}) \leftarrow {\tt env.step}  (a)$ and record $(s_t, s_{t+1}, a_t^m, R_t^{\prime})$ in $\mathcal{D}$
\ENDFOR
\STATE update $\theta_{old} \leftarrow \theta$ using $\mathcal{D}$ by PPO algorithm
\ENDFOR
\end{algorithmic}
\end{algorithm}

\phantomsection
\label{cite_this}
\para{Constructing Mixed Initial State Distribution.} 
With the state mask $\tilde{\pi}$, we construct a mixed initial state distribution to expand the coverage of the state space. Initially, we randomly sample a trajectory by executing the pre-trained policy $\pi$. Subsequently, the state mask is applied to pinpoint the most important state within the episode $\tau$ by assessing the significance of each state. The resulting distribution of these identified critical states is denoted as $d_{\rho}^{\hat{\pi}}(s)$. Indeed, in \autoref{subsec:theory}, we demonstrate that this re-weighting-like sampling is equivalent to sampling the state from a better policy $\hat{\pi}$. We then set the initial distribution $\mu$ as a mixture of the selected important states distribution $d_{\rho}^{\hat{\pi}}(s)$ and the original initial distribution of interest $\rho$: $\mu(s) = \beta d_{\rho}^{\hat{\pi}}(s) + (1 - \beta)\rho(s)$, where $\beta$ is a hyper-parameter.

\para{Exploration with Random Network Distillation.}
 Starting from the new initial state distribution, we continue training the DRL agent while encouraging the agent to do exploration. In contrast to goal-conditional RL ~\cite{ren2019exploration, ecoffet2019go}, which typically involve random exploration from identified frontiers, we advocate for the RL agent to explore novel states to increase the state coverage. Motivated by this, we adopt Random Network Distillation (RND)~\cite{burda2018exploration} which is proved to be an effective exploration bonus, especially in large and continuous state spaces where count-based bonuses~\cite{bellemare2016unifying, ostrovski2017count} can be hard to extend. Specifically, we directly utilize the PPO algorithm to update the policy $\pi$, except that we add the intrinsic reward to the task reward, that is, we optimize $R'(s_t, a_t) = R(s_t, a_t) + \lambda |f(s_{t+1}) - \hat{f}(s_{t+1})|^2$, where $\lambda$ controls the trade-off between the task reward and exploration bonus. Along with the policy parameters, the RND predictor network $\hat{f}$ is updated to regress to the target network $f$. Note that, as the state coverage increases, RND bonuses decay to zero and a performed policy is recovered. We present our proposed refining method in \autoref{alg:retrain}.

\begin{algorithm}[t]
\caption{Refining the DRL Agent.}
\label{alg:retrain}
\footnotesize
\begin{algorithmic}
\STATE {\bfseries Input:} Pre-trained policy $\pi$, corresponding state mask $\tilde{\pi}$, default initial state distribution $\rho$, reset probability threshold $p$
\STATE {\bfseries Output:} The agent's policy after refining $\pi^{\prime}$
\FOR{iteration=1, 2, \dots}
\STATE $\mathcal{D} \leftarrow \emptyset$
\STATE RAND\_NUM $\leftarrow$ RAND(0,1)\\
\IF{RAND\_NUM $<$ $p$}
\STATE Run $\pi$ to obtain a trajectory $\tau$ of length $K$ \\
\STATE Identify the most critical state $s_t$ in $\tau$ via state mask $\tilde{\pi}$\\
\STATE Set the initial state $s_0 \leftarrow s_t$\\
\ELSE
\STATE Set the initial state $s_0 \sim \rho$ \\
\ENDIF
\FOR{t=0 to $T$}
\STATE Sample $a_t \sim \pi(a_t | s_t)$\\
\STATE $(s_{t+1}, R_t) \leftarrow {\tt env.step}  (a_t)$
\STATE Calculate RND bonus $R_t^{RND}=\left\|f\left(s_{t+1}\right)-\hat{f}\left(s_{t+1}\right)\right\|^2$ with normalization
\STATE Add $(s_t, s_{t+1}, a_t, R_t + \lambda R_t^{RND})$ to $\mathcal{D}$ 
\ENDFOR
\STATE Optimize $\pi_\theta$ w.r.t PPO loss on $\mathcal{D}$
\STATE Optimize ${\hat{f}}_\theta$ w.r.t. MSE loss on $\mathcal{D}$ using Adam
\ENDFOR
\STATE $\pi^{\prime} \leftarrow \pi_{\theta}$
\end{algorithmic}
\end{algorithm}

\subsection{Theoretical Analysis}
\label{subsec:theory}
Finally, we provide theoretical analysis demonstrating that our refining algorithm can tighten the sub-optimality gap: $\text{SubOpt} := V^{\pi^*}(\rho) - V^{\pi'}(\rho)$, (\ie the gap between the long-term reward collected by the optimal policy $\pi^*$ and that obtained by the refined policy $\pi'$ when starting from the default initial state distribution $\rho$).

In particular, we aim to answer the following two questions: 
\textbf{Q1:}{\em What are the benefits of incorporating StateMask to determine the exploration frontier?} 

\textbf{Q2:} {\em what advantages does starting exploration from the mixed initial distribution offer?}

To answer the questions, we first show that determining the exploration frontiers  based on StateMask is equivalent to sampling states from a better policy compared to $\pi$. Then, we demonstrate that under the mixed initial distribution as introduced in \autoref{sec:tech}, we could provide a tighter upper bound for the sub-optimality of trained policy $\pi$ compared with randomly selecting visited states to form the initial distribution.

In order to answer \textbf{Q1}, we begin with \sref{Assumption}{assumption:better} to assume the relationship between the policy value and the state distribution mismatch coefficient.

\begin{assumption}
\label{assumption:better}
For two polices $\pi$ and $\hat{\pi}$, if $\eta(\hat{\pi}) \geq \eta(\pi)$, then we have $\left\|\frac{d_\rho^{\pi^*}}{d_{\rho}^{\hat{\pi}}}\right\|_{\infty} \leq \left\|\frac{d_\rho^{\pi^*}}{d_{\rho}^{\pi}}\right\|_{\infty}$.
\end{assumption}
Intuitively, this assumption posits that a superior policy would inherently possess a greater likelihood of visiting all favorable states. We give validation of this assumption in a 2-state MDP in \sref{Appendix}{appendix:2state}.

We further present \sref{Lemma}{lemma:sampling} to answer \textbf{Q1}, \ie the benefits of incorporating StateMask to determine the exploration frontier. The proof of \sref{Lemma}{lemma:sampling} can be found in \sref{Appendix}{appendix:lemma_q1}.

\begin{lemma}
\label{lemma:sampling}
Given a pre-trained policy $\pi$, our MaskNet-based sampling approach in \autoref{sec:tech} is equivalent to sampling states from a state occupation distribution induced by an improved policy $\hat{\pi}$.
\end{lemma}

In order to answer \textbf{Q2}, we start with presenting \sref{Theorem}{theorem:performance_diff} to bound the sub-optimality via the state distribution mismatch coefficient. 

\begin{theorem}
Assume that for the refined policy $\pi^{\prime}$,  $\mathbb{E}_{s \sim d_{\mu}^{\pi'}}\left[\max _a A^{\pi'}(s, a)\right] < \epsilon$. For two initial state distributions $\mu$ and $\rho$, we have the following bound~\cite{kakade2002approximately}
\vspace{-5pt}
\begin{equation}
\small
V^{\pi^*}(\rho)-V^{\pi'}(\rho) 
\leq \mathcal{O}(\frac{\epsilon}{(1-\gamma)^2} \left\|\frac{d_\rho^{\pi^*}}{d_{\rho}^{\hat{\pi}}}\right\|_{\infty}).
\end{equation}
\label{theorem:performance_diff}
\vspace{-15pt}
\end{theorem}

The proof of \sref{Theorem}{theorem:performance_diff} can be found in \sref{Appendix}{appendix:proof_theorem_1}. It indicates that the upper bound on the difference between the performance of the optimal policy $\pi^*$ and that of the policy $\pi'$ after refining is proportional to $\left\|\frac{d_\rho^{\pi^*}}{d_{\rho}^{\hat{\pi}}}\right\|_{\infty}$. 
With \sref{Assumption}{assumption:better} and \sref{Lemma}{lemma:sampling}, we now claim that our refining method with our explanation could further tighten the sub-optimality gap via reducing the distribution mismatch coefficient compared with forming an initial distribution by random selecting visited states, \ie with a random explanation.

\begin{claim}
We can form a better (mixed) initial state distribution $\mu$ with our explanation method and tighten the upper bound of $V^{\pi^*}(\rho)-V^{\pi'}(\rho)$ compared with random explanation.
\label{claim:tight_bound}
\end{claim}

The details of the analysis can be found in \sref{Appendix}{appendix:claim_1}. Based on \sref{Assumption}{assumption:mismatch} and \sref{Claim}{claim:tight_bound}, we can learn to perform as well as the
optimal policy as long as the
visited states of the optimal policy are covered by the (mixed) initial distribution.



\section{Evaluation}
\label{sec:eval}
In this section,  we start with our experiment setup and design, followed by experiment results and analysis. We provide additional evaluation details in \autoref{appendix:evaluation}.


\subsection{Experiment Setup}
\label{sec:setup}

\para{Environment Selection.} We select eight representative environments to demonstrate the effectiveness of {\tt RICE} across two categories: simulated games (Hopper, Walker2d, Reacher, and HalfCheetah of the MuJoCo games) and real-world applications (selfish mining, network defense, autonomous driving, and malware mutation)~\footnote{These are representative security applications that have a significant impact on the security community~\cite{anderson2018learning} and they represent RL tasks with sparse rewards, which are common in security applications.}. We additionally run the experiments in the three sparse MuJoCo games introduced by \citet{mazoure2019leveraging}. The details of these applications can be found in \sref{Appendix}{appendix:app_description}.

\para{Baseline Explanation Methods.} 
Since our explanation method proposes an alternative design of StateMask, the first baseline is StateMask. We compare our explanation method with StateMask to show the equivalence and efficiency of our method. Additionally, we introduce ``Random'' as a baseline explanation method. ``Random'' identifies critical steps by randomly selecting a visited state as the critical state. 

\para{Baseline Refining Methods.} We compare our refining method with three baselines. The first baseline is ``PPO fine-tuning''~\cite{schulman2017proximal}, \ie lowering the learning rate and continuing training with the PPO algorithm. The second baseline is a refining method introduced by StateMask~\cite{cheng2023statemask}, \ie resetting to the critical state and continuing training from the critical state. The third baseline is Jump-Start Reinforcement Learning (referred to as ``JSRL'')~\cite{uchendu2023jump}. JSRL introduces a guided policy $\pi_g$ to set up a curriculum to train an exploration policy $\pi_e$. Through initializing $\pi_e = \pi_g$, we can transform JSRL to be a refining method that can further improve the performance of the guided policy.

\para{Evaluation Metrics.} 
To evaluate the fidelity of the generated explanation, we utilize an established fidelity score metric defined in StateMask~\cite{cheng2023statemask}. The idea is to use a sliding window to step through all time steps and then choose the window with the highest average importance score (scored by the explanation method). The width of the sliding window is $l$ while the whole trajectory length is $L$. Then we randomize the action(s) at the selected critical step(s) in the selected window (\ie masking) and measure the average reward change as $d$. Additionally, we denote the maximum possible reward change as $d_{max}$. Therefore, the fidelity score is calculated as $\log(d/d_{max})-\log(l/L)$. A higher fidelity score indicates higher fidelity.

For the applications with dense rewards except the malware mutation application, we measure the reward of the target agent before and after refining. In the case of the malware mutation application, we report the ``final reward'' as the probability of evading the malware detector, both before and after refining. For the applications with sparse rewards, we report the performance during the refining process.

\subsection{Experiment Design}

We use the following experiments to evaluate 
the fidelity and efficiency of the explanation method,
the effectiveness of the refining method and other factors that influenced the system performance (\eg alternative design choices, hyper-parameters).

\para{Experiment I.} 
To show the equivalence of our explanation method with StateMask, we compare the fidelity of our method with StateMask. Given a trajectory, the explanation method first identifies and ranks top-$K$ important time steps. An accurate explanation means the important time steps have significant contributions to the final reward. To validate this, we let the agent fast-forward to the critical step and force the target agent to take random actions. Then we follow the target agent's policy to complete the rest of the time steps. If the explanation is accurate, we expect a major change to the final reward by randomizing the actions at the important steps. We compute the fidelity score of each explanation method as mentioned in StateMask across 500 trajectories. We set $K=10\%, 20\%, 30\%, 40\%$ and report the fidelity of the selected methods under each setup. We repeat each experiment 3 times with various random seeds and report the mean and standard deviation. 
Additionally, to show the efficiency of our design, we report the training time of the mask network using StateMask and our method when given a fixed number of training samples.

\para{Experiment II.} To show the effectiveness of the refining method, we compare the agent's performance after refining using our method and three aforementioned baseline methods, \ie PPO fine-tuning~\cite{schulman2017proximal}, StateMask's fine-tuning from critical steps~\cite{cheng2023statemask}, and Jump-Start Reinforcement Learning~\cite{uchendu2023jump}. For this experiment, all the refining methods use the {\em same explanation} generated by our explanation method if needed, to ensure a fair comparison. Additionally, we conduct a qualitative study to understand how our refining method influences agent behavior and performance. 

\para{Experiment III} To investigate how the quality of explanation affects the downstream refining process, we run our proposed refining method based on the critical steps identified by different explanation methods (Random, StateMask, and our method) and compare the agent's performance after refining.

\para{Experiment IV.} To show the versatility of our method, we examine the refining performance when the pre-trained agent was trained by other algorithms such as Soft Actor-Critic (SAC)~\cite{haarnoja2018soft}. First, we obtain a pre-trained SAC agent and then use Generative Adversarial Imitation Learning (GAIL)~\cite{ho2016generative} to learn an approximated policy network. We compare the refining performance using our method against baseline methods, \ie PPO fine-tuning~\cite{schulman2017proximal}, StateMask's fine-tuning from critical steps~\cite{cheng2023statemask}, and Jump-Start Reinforcement Learning~\cite{uchendu2023jump}. In addition, we also include fine-tuning the pre-trained SAC agent with the SAC algorithm as a baseline.

\para{Experiment V.} We test the impact of hyper-parameter choices for two primary hyper-parameters for refining: $p$ (used to control the mixed initial state distribution) and $\lambda$ (used to control the exploration bonus). For our refining method, we vary $p$ from $\{0, 0.25, 0.5, 0.75, 1\}$ and vary $\lambda$ from $\{0, 0.1, 0.01, 0.001\}$. By examining the agent's performance with various $\lambda$ values, we can further investigate the necessity of the exploration bonus. Additionally, we evaluate the choice of $\alpha$ for our explanation method (used to control the mask ratio for the mask network). Specifically, we vary $\alpha$ from $\{0.01, 0.001, 0.0001\}$.


\subsection{Experiment Results}
\label{sec:result}


\para{Fidelity and Efficiency of Explanation.} 
We compare the fidelity scores of our method with StateMask in all applications and provide the full results in \autoref{fig:fid_score} of \sref{Appendix}{appendix:exp-add}. We observe that the fidelity scores of StateMask and our method are comparable. Furthermore, We evaluate the efficiency of our explanation method compared with StateMask. We report the cost time and the number of samples when training our explanation method and StateMask in \autoref{table:efficiency} of \sref{Appendix}{appendix:exp-add}. We observe an average of 16.8\% drop in the training time compared with StateMask. The reason is that the training algorithm of the mask network in StateMask involves an estimation of the discounted accumulated reward with respect to the current policy of the perturbed agent and the policy of the target agent which requires additional computation cost. In contrast, our design only adds an additional term to the reward which is simple but effective.

\begin{table*}[t]
\footnotesize
\centering
\caption{
{\bf Agent Refining Performance}---%
{\small
``No Refine'' indicates the target agent's performance before refining. For the first group of experiments (left), we fixed the explanation method to our method (mask network) and varied the refining methods. For the second group of experiments (right), we fixed the refining method to our method and varied the explanation methods. We report the mean value (standard deviations) of the final reward after refining. A higher value is better. 
}
}
\resizebox{\textwidth}{!}{ 
\begin{tabular}{c?c?cccc?ccc}
\Xhline{1pt}
\multirow{2}{*}{Task} & \multirow{2}{*}{No Refine} & \multicolumn{4}{c?}{Fix Explanation; Vary Refine Methods}                                       & \multicolumn{3}{c}{Fix Refine; Vary Explanation Methods}                                        \\ \cline{3-9} 
&    & \multicolumn{1}{c|}{PPO}  & \multicolumn{1}{c|}{JSRL} & \multicolumn{1}{c|}{StateMask-R}  &  \textbf{Ours}            
& \multicolumn{1}{c|}{Random}   & \multicolumn{1}{c|}{StateMask}     & \textbf{Ours} \\ \Xhline{1pt}
Hopper        & 3559.44 (19.15)                & \multicolumn{1}{c|}{3638.75 (16.67)}  & \multicolumn{1}{c|}{3635.08 (9.82)}  & \multicolumn{1}{c|}{3652.06 (8.63)}  & \textbf{3663.91 (20.98)}  & \multicolumn{1}{c|}{3648.98 (39.06)}  & \multicolumn{1}{c|}{3661.86 (19.95)}   & \textbf{3663.91 (20.98)}  \\ \hline
Walker2d        & 3768.79 (18.68)      & \multicolumn{1}{c|}{3965.63 (9.46)}          & \multicolumn{1}{c|}{3963.57 (6.73)}  & \multicolumn{1}{c|}{3966.96 (3.39)}  & \textbf{3982.79 (3.15)}  & \multicolumn{1}{c|}{3969.64(6.38)}  & \multicolumn{1}{c|}{3982.67 (5.55)}   & \textbf{3982.79 (3.15)}  \\ \hline
Reacher        & -5.79 (0.73)   & \multicolumn{1}{c|}{-3.04 (0.04)}             & \multicolumn{1}{c|}{-3.23 (0.26)}  & \multicolumn{1}{c|}{-3.45 (0.32)}  & \textbf{-2.66 (0.03)}  & \multicolumn{1}{c|}{-3.11 (0.42)}  & \multicolumn{1}{c|}{-2.69 (0.28)}  & \textbf{-2.66 (0.03)}  \\ \hline
HalfCheetah        & 2024.09 (28.34)    & \multicolumn{1}{c|}{2133.31 (4.11)}            & \multicolumn{1}{c|}{2128.04 (0.91)}  & \multicolumn{1}{c|}{2085.28 (1.92)}  & \textbf{2138.89 (3.22)}  & \multicolumn{1}{c|}{2132.01 (0.76)}  & \multicolumn{1}{c|}{2136.23 (0.49)}    & \textbf{2138.89 (3.22)}  \\ \hline
Selfish Mining        & 14.36 (0.24)                & \multicolumn{1}{c|}{14.93 (0.45)}  & \multicolumn{1}{c|}{14.88 (0.51)}  & \multicolumn{1}{c|}{14.53 (0.33)}  & \textbf{16.56 (0.63)}  & \multicolumn{1}{c|}{15.09 (0.28)}    & \multicolumn{1}{c|}{16.49 (0.46)} & \textbf{16.56 (0.63)}  \\ \hline
Cage Challenge 2      & -23.64 (0.27)               & \multicolumn{1}{c|}{-23.58 (0.37)} & \multicolumn{1}{c|}{-22.97 (0.57)}  & \multicolumn{1}{c|}{-26.98 (0.84)} & \textbf{-20.02 (0.32)} & \multicolumn{1}{c|}{-25.94 (2.34)}    & \multicolumn{1}{c|}{-20.07 (1.33)} & \textbf{-20.02 (0.32)} \\ \hline
Auto Driving          & 10.30 (2.25)                & \multicolumn{1}{c|}{13.37 (3.10)}  & \multicolumn{1}{c|}{11.26 (3.66)}  & \multicolumn{1}{c|}{7.62 (1.77)}   & \textbf{17.03 (1.65)}  & \multicolumn{1}{c|}{11.72 (1.77)}    & \multicolumn{1}{c|}{16.28 (2.33)}& \textbf{17.03 (1.65)}  \\ \hline
Malware Mutation      & 42.20 (6.86)                & \multicolumn{1}{c|}{49.33 (8.59)}  & \multicolumn{1}{c|}{43.10 (7.24)}  & \multicolumn{1}{c|}{50.13 (8.14)}  & \textbf{57.53 (8.71)}  & \multicolumn{1}{c|}{48.60 (7.60)}    & \multicolumn{1}{c|}{57.16 (8.51)}& \textbf{57.53 (8.71)}  \\ \Xhline{1pt}
\end{tabular}
}
\label{table:retrain_method}
\end{table*}

\para{Effectiveness of Refining.} We compare the agent's performance after refining using different retaining methods across all applications with dense rewards in \autoref{table:retrain_method}. The performance is measured by the final reward of the refined agent. In most applications, rewards are typically assigned positive values. However, in Cage Challenge 2, the reward is designed to incorporate negative values (see \sref{Appendix}{appendix:app_description}). We have three main observations.  
First, we observe that our refining method can bring the largest improvement for the target agent in all applications. 
Second, we find that the PPO fine-tuning method only has marginal improvements for the agents due to its incapability of jumping out of local optima. Third, the refining method proposed in StateMask (which is to start fine-tuning only from critical steps) cannot always improve the agent's performance. The reason is that this refining strategy can cause overfitting and thus harm the agent's performance. We illustrate this problem in greater detail in a case study of Malware Mutation in \autoref{appendix:malware}. It is also worth mentioning that we discover design flows of Malware Mutation and present the details in \autoref{appendix:malware}.

We also run our experiments of varying refining methods on selected MuJoCo games with sparse rewards. \autoref{fig:sparse_games} shows the results of our method against other baselines in SparseHopper and SparseHalfCheetah games. We observe that our refining method has significant advantages over other baselines with respect to final performance and refining efficiency. Through varying explanation methods, we confirm that the contribution should owe to our explanation method. We leave the refining results of the SparseWalker2d game and the hyper-parameter sensitivity results of all sparse MuJoCo games in \sref{Appendix}{appendix:sparse_games}. 

In addition to numerical results, we also provide a qualitative analysis of the autonomous driving case to understand how {\tt RICE} influences agent behavior and performance, particularly in a critical state, in \sref{Appendix}{appendix:qualitative}. We visualize the agent's behavior before and after refining the agent to show that {\tt RICE} is able to help the agent break through the bottleneck based on the identified critical states of the failure. 

\begin{figure*}[t]
\centering
\includegraphics[width=\textwidth]{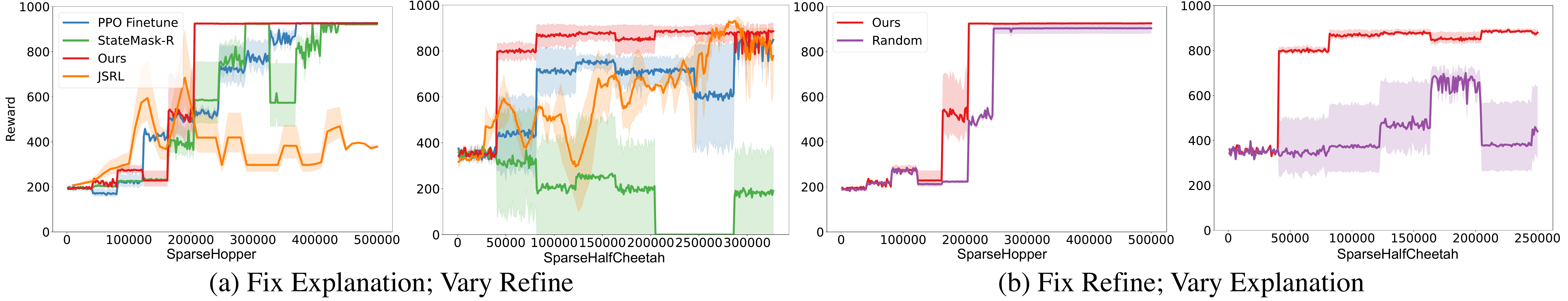}
\caption{\textbf{Agent Refining Performance in two Sparse MuJoCo Games}---%
{\small For Group (a), we fix the explanation method to our method (mask network) if needed while varying refining methods. For Group (b), we fix the refining method to our method while varying the explanation methods. }}
\label{fig:sparse_games}
\vspace{-5pt}
\end{figure*}

\para{Refining based on Different Explanations.} 
To examine how the quality of explanation affects the downstream refining process, we present \autoref{table:retrain_method}. We run our proposed refining method based on the critical steps identified by ours and Random. We have two main observations. First, using the explanation generated by our mask network, the refining achieves the best outcome across all applications. 
Second, using the explanation generated by our explanation significantly outperforms the random baseline. This aligns with our theoretical analysis that our refining framework provides a tighter bound for the sub-optimality.  



\para{Refining a Pre-trained Agent of Other Algorithms.} To show that our framework is general to refine pre-trained agents that were not trained by PPO algorithms, we do experiments on refining a SAC agent in the Hopper game. \autoref{fig:sac} demonstrates the advantage of our refining method against other baselines when refining a SAC agent. 
Additionally, we observe that fine-tuning the DRL agent with the SAC algorithm still suffers from the training bottleneck while switching to the PPO algorithm provides an opportunity to break through the bottleneck. 
We provide the refining curves when varying hyper-parameters $p$ and $\lambda$ in \sref{Appendix}{appendix:exp-add}.

\begin{figure}[t]
\centering
\includegraphics[width=0.48\textwidth]{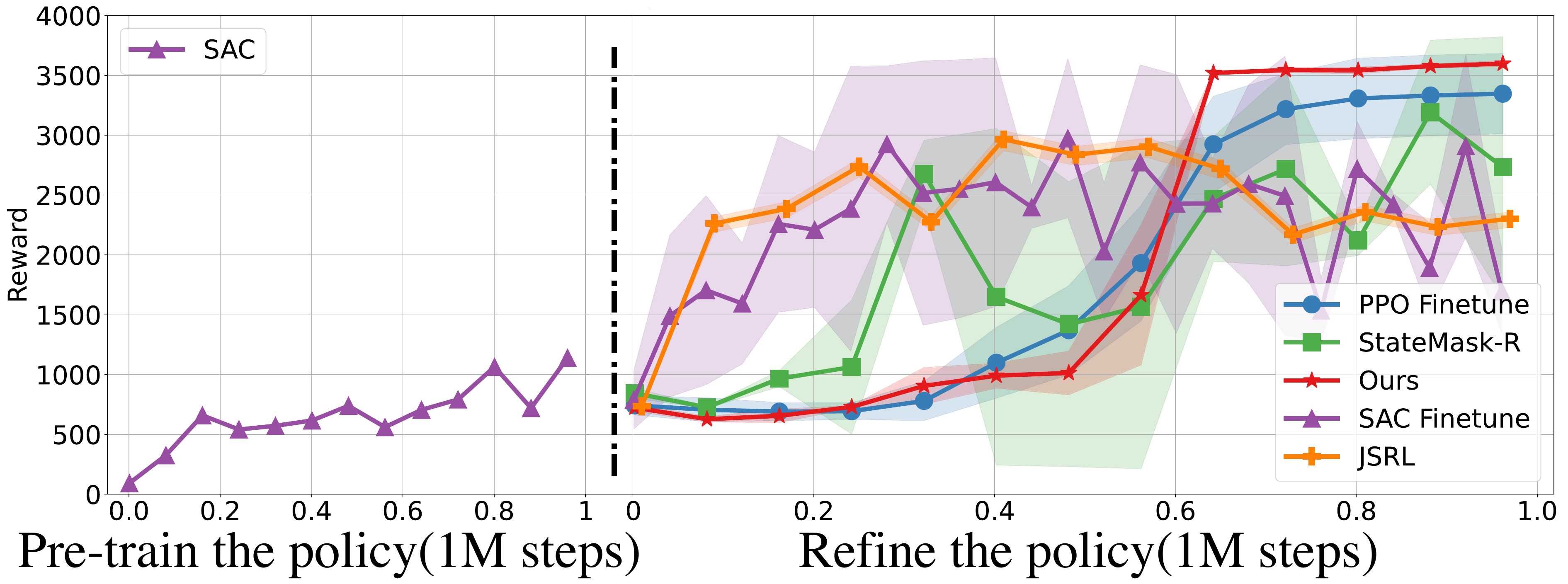}
\caption{\textbf{SAC Agent Refining Performance in Hopper Game} ---%
{\small In the left part, we show the training curve of obtaining a pre-trained policy through the SAC algorithm. In the right part, we show the refining curves of different methods.}}
\label{fig:sac}
\vspace{-10pt}
\end{figure}

\para{Impact of Hyper-parameters.} 
Due to space limit, we provide the sensitivity of hyper-parameters $p$, $\lambda$, and $\alpha$ in \sref{Appendix}{appendix:exp-add}. We have three main observations.


First, $p$ controls the mixing ratio of critical states (identified by the explanation method) and the initial state distribution for refining. The performance is low when $p=0$ (all starting from the default initial distribution) or $p=1$ (all starting from the identified critical states).  The performance has significant improvements when $0 < p < 1$, \ie using a mixed initial state distribution. Across all applications, we observe that setting $p$ to 0.25 or 0.5 is most beneficial. A mixed initial distribution can help eliminate the problem of overfitting. 

Second, as long as $\lambda>0$ (thereby enabling exploration), there is a noticeable improvement in performance, highlighting the importance of exploration in refining the pre-trained agent. The result is less sensitive to the specific value of $\lambda$. In general, a $\lambda$ value of 0.01 yields good performance across all four applications.

Third, recall that the hyper-parameter $\alpha$ is to control the bonus of blinding the target agent when training the mask network. We vary $\alpha$ from $\{0.01, 0.001, 0.0001\}$ and find that our explanation method is not that sensitive to $\alpha$. 


\section{Discussion}
\label{sec:discussion}

\para{Applicability.} {\tt RICE} is suitable for DRL applications that are trained within {\em controllable environment} (\eg simulators), in order to generate explanations. 
In fact, most of today's DRL applications rely on some form of simulator for their training. For example, for safety-critical applications such as autonomous driving, DRL agents are usually designed, trained, and tested in a simulated environment first before moving them to real-world testing. Simulation platforms broadly include Carla~\cite{Dosovitskiy17} and MetaDrive~\cite{li2022metadrive} which have been used to facilitate the training of DRL agents~\cite{zhang2021end, wang2023autonomous, peng2022safe}. Therefore, {\tt RICE} should be applicable to such DRL systems (especially during their development phase) for refining a pre-trained DRL agent.  


\para{Warm Start vs Cold Start.} As is mentioned in \autoref{sec:bg}, our method requires a ``warm start'' setting, \ie the agent has good coverage of the state distribution of the optimal policy. Even if the agent has good coverage of the state distribution, it does not necessarily mean that the agent has already learned a good policy due to the potential of choosing wrong actions~\cite{uchendu2023jump}. Therefore, the training bottleneck can still exist under a good coverage of the state distribution. In contrast, Our method does not work well in a “cold start” setting, \ie when the state coverage of the pre-trained policy is extremely poor. In that case, step-level explanation methods cannot give useful help and our method is actually equivalent to the RND method~\footnote{We provide an example of Mountain Car game in \sref{Appendix}{appendix:limitation} to illustrate this limitation.}.

\para{Critical State Filtering.} Though {\tt RICE} identifies critical states based on their necessity for achieving good outcomes, it does not fully consider their importance for further agent learning. For instance, a state might be deemed critical, yet the trained agent could have already converged to the optimal action for that state. In such cases, resetting the environment to this state doesn't significantly benefit the learning process. Future work could explore additional filtering of critical states using metrics such as policy convergence or temporal difference (TD) errors, which may help concentrate efforts and accelerate refinement.

\section{Conclusion}
In this paper, we present {\tt RICE} to break through bottlenecks of reinforcement learning training with explanation. We propose an alternative design of StateMask to provide high-fidelity explanations for DRL agents' behaviors, by identifying critical time steps that contribute the most to the agent's success/failure. We encourage the agent to explore starting from a mixture of default initial states and the identified critical states. Compared with existing refining strategies, we empirically show that our method brings the largest improvement after refining with theoretical guarantees.

\section*{Acknowledgements}

This project was supported in part by Northwestern University TGS Fellowship and NSF Grants 2225234, 2225225, 2229876, 1955719, and 2055233.

\section*{Impact Statement}

This paper presents work whose goal is to advance the field of reinforcement learning with explanation. There are many potential social impacts of our work. Our approach provides a feasible solution to break through the training bottlenecks of reinforcement learning with explanation, which is an automatic process and saves manual effort.

However, it is also worth noting the potential negative societal impacts of our work. Some of the real-world applications we select such as malware mutation can create attack examples that may bring additional ethical concerns. In the realm of security research, the ultimate goal of these tasks is to generate stronger testing cases to enhance the defense, and it is standard practice. Take malware mutation as an example, the produced samples can be used to proactively improve the robustness and effectiveness of malware detection systems (\eg through adversarial training), thereby benefiting cybersecurity defense~\cite{yang2017malware}.

\bibliography{ref}

\begin{thebibliography}{57}
\providecommand{\natexlab}[1]{#1}
\providecommand{\url}[1]{\texttt{#1}}
\expandafter\ifx\csname urlstyle\endcsname\relax
  \providecommand{\doi}[1]{doi: #1}\else
  \providecommand{\doi}{doi: \begingroup \urlstyle{rm}\Url}\fi

\bibitem[git({\natexlab{a}})]{githubGitHubBfilarmalware_rl}
{G}it{H}ub - bfilar/malware\_rl: {M}alware {B}ypass {R}esearch using {R}einforcement {L}earning --- github.com.
\newblock \url{https://github.com/bfilar/malware\_rl}, {\natexlab{a}}.

\bibitem[git({\natexlab{b}})]{githubGitHubCagechallengecagechallenge2}
{G}it{H}ub - cage-challenge/cage-challenge-2: {T}{T}{C}{P} {C}{A}{G}{E} {C}hallenge 2 --- github.com.
\newblock \url{https://github.com/cage-challenge/cage-challenge-2}, {\natexlab{b}}.

\bibitem[git({\natexlab{c}})]{githubGitHubJohncardiffcyborgcage2}
{G}it{H}ub - john-cardiff/-cyborg-cage-2 --- github.com.
\newblock \url{https://github.com/john-cardiff/-cyborg-cage-2}, {\natexlab{c}}.

\bibitem[git({\natexlab{d}})]{githubGitHubRoibarzurptoselfishmining}
{G}it{H}ub - roibarzur/pto-selfish-mining: {C}ode repository for technical papers about selfish mining analysis. --- github.com.
\newblock \url{https://github.com/roibarzur/pto-selfish-mining}, {\natexlab{d}}.

\bibitem[mou()]{mountaincar}
Mountain car continuous.
\newblock \url{https://mgoulao.github.io/gym-docs/environments/classic_control/mountain_car_continuous/}.
\newblock Accessed: 2024-05-24.

\bibitem[Agarwal et~al.(2020)Agarwal, Henaff, Kakade, and Sun]{agarwal2020pc}
Agarwal, A., Henaff, M., Kakade, S., and Sun, W.
\newblock Pc-pg: Policy cover directed exploration for provable policy gradient learning.
\newblock \emph{Proc. of NeurIPS}, 2020.

\bibitem[Agarwal et~al.(2021)Agarwal, Kakade, Lee, and Mahajan]{agarwal2021theory}
Agarwal, A., Kakade, S.~M., Lee, J.~D., and Mahajan, G.
\newblock On the theory of policy gradient methods: Optimality, approximation, and distribution shift.
\newblock \emph{Journal of Machine Learning Research}, 2021.

\bibitem[Agarwal et~al.(2022)Agarwal, Schwarzer, Castro, Courville, and Bellemare]{agarwal2022reincarnating}
Agarwal, R., Schwarzer, M., Castro, P.~S., Courville, A.~C., and Bellemare, M.
\newblock Reincarnating reinforcement learning: Reusing prior computation to accelerate progress.
\newblock In \emph{Proc. of NeurIPS}, 2022.

\bibitem[Anderson et~al.(2018)Anderson, Kharkar, Filar, Evans, and Roth]{anderson2018learning}
Anderson, H.~S., Kharkar, A., Filar, B., Evans, D., and Roth, P.
\newblock Learning to evade static pe machine learning malware models via reinforcement learning.
\newblock \emph{arXiv preprint arXiv:1801.08917}, 2018.

\bibitem[Bar-Zur et~al.(2023)Bar-Zur, Abu-Hanna, Eyal, and Tamar]{bar2022werlman}
Bar-Zur, R., Abu-Hanna, A., Eyal, I., and Tamar, A.
\newblock Werlman: To tackle whale (transactions), go deep (rl).
\newblock In \emph{Proc. of IEEE S\&P}, 2023.

\bibitem[Bellemare et~al.(2016)Bellemare, Srinivasan, Ostrovski, Schaul, Saxton, and Munos]{bellemare2016unifying}
Bellemare, M., Srinivasan, S., Ostrovski, G., Schaul, T., Saxton, D., and Munos, R.
\newblock Unifying count-based exploration and intrinsic motivation.
\newblock In \emph{Proc. of NeurIPS}, 2016.

\bibitem[Burda et~al.(2018)Burda, Edwards, Storkey, and Klimov]{burda2018exploration}
Burda, Y., Edwards, H., Storkey, A., and Klimov, O.
\newblock Exploration by random network distillation.
\newblock In \emph{Proc. of ICLR}, 2018.

\bibitem[CAGE(2022)]{cage_challenge_2_announcement}
CAGE.
\newblock Ttcp cage challenge 2.
\newblock In \emph{Proc. of AAAI-22 Workshop on Artificial Intelligence for Cyber Security (AICS)}, 2022.

\bibitem[Cai et~al.(2022)Cai, Ding, Chen, Jiang, Sugiyama, and Zhou]{cai2022seeing}
Cai, X.-Q., Ding, Y.-X., Chen, Z., Jiang, Y., Sugiyama, M., and Zhou, Z.-H.
\newblock Seeing differently, acting similarly: Heterogeneously observable imitation learning.
\newblock In \emph{Proc. of ICLR}, 2022.

\bibitem[Cai et~al.(2023)Cai, Zhang, Chiang, and Sugiyama]{cai2023imitation}
Cai, X.-Q., Zhang, Y.-J., Chiang, C.-K., and Sugiyama, M.
\newblock Imitation learning from vague feedback.
\newblock In \emph{Proc. of NeurIPS}, 2023.

\bibitem[Chang et~al.(2023)Chang, Brantley, Ramamurthy, Misra, and Sun]{chang2023learning}
Chang, J.~D., Brantley, K., Ramamurthy, R., Misra, D., and Sun, W.
\newblock Learning to generate better than your llm.
\newblock \emph{arXiv preprint arXiv:2306.11816}, 2023.

\bibitem[Cheng et~al.(2023)Cheng, Wu, Yu, Sun, Guo, and Xing]{cheng2023statemask}
Cheng, Z., Wu, X., Yu, J., Sun, W., Guo, W., and Xing, X.
\newblock Statemask: Explaining deep reinforcement learning through state mask.
\newblock In \emph{Proc. of NeurIPS}, 2023.

\bibitem[Dosovitskiy et~al.(2017)Dosovitskiy, Ros, Codevilla, Lopez, and Koltun]{Dosovitskiy17}
Dosovitskiy, A., Ros, G., Codevilla, F., Lopez, A., and Koltun, V.
\newblock {CARLA}: {An} open urban driving simulator.
\newblock In \emph{Proc. of CoRL}, pp.\  1--16, 2017.

\bibitem[drive Contributors(2021)]{didrive}
drive Contributors, D.
\newblock {DI-drive: OpenDILab} decision intelligence platform for autonomous driving simulation.
\newblock \url{https://github.com/opendilab/DI-drive}, 2021.

\bibitem[Ecoffet et~al.(2019)Ecoffet, Huizinga, Lehman, Stanley, and Clune]{ecoffet2019go}
Ecoffet, A., Huizinga, J., Lehman, J., Stanley, K.~O., and Clune, J.
\newblock Go-explore: a new approach for hard-exploration problems.
\newblock \emph{arXiv preprint arXiv:1901.10995}, 2019.

\bibitem[Ecoffet et~al.(2021)Ecoffet, Huizinga, Lehman, Stanley, and Clune]{ecoffet2021first}
Ecoffet, A., Huizinga, J., Lehman, J., Stanley, K.~O., and Clune, J.
\newblock First return, then explore.
\newblock \emph{Nature}, 2021.

\bibitem[Erez et~al.(2012)Erez, Tassa, and Todorov]{erez2012infinite}
Erez, T., Tassa, Y., and Todorov, E.
\newblock Infinite-horizon model predictive control for periodic tasks with contacts.
\newblock \emph{Robotics: Science and systems VII}, pp.\ ~73, 2012.

\bibitem[Eyal \& Sirer(2018)Eyal and Sirer]{eyal2018majority}
Eyal, I. and Sirer, E.~G.
\newblock Majority is not enough: Bitcoin mining is vulnerable.
\newblock \emph{Communications of the ACM}, 61\penalty0 (7):\penalty0 95--102, 2018.

\bibitem[Eysenbach et~al.(2021)Eysenbach, Salakhutdinov, and Levine]{eysenbach2021information}
Eysenbach, B., Salakhutdinov, R., and Levine, S.
\newblock The information geometry of unsupervised reinforcement learning.
\newblock In \emph{Proc. of ICLR}, 2021.

\bibitem[Fu et~al.(2018)Fu, Luo, and Levine]{fu2018learning}
Fu, J., Luo, K., and Levine, S.
\newblock Learning robust rewards with adverserial inverse reinforcement learning.
\newblock In \emph{Proc. of ICLR}, 2018.

\bibitem[Guan et~al.(2021)Guan, Verma, Guo, Zhang, and Kambhampati]{guan2021widening}
Guan, L., Verma, M., Guo, S.~S., Zhang, R., and Kambhampati, S.
\newblock Widening the pipeline in human-guided reinforcement learning with explanation and context-aware data augmentation.
\newblock In \emph{Proc.of NeurIPS}, 2021.

\bibitem[Guo et~al.(2021)Guo, Wu, Khan, and Xing]{guo2021edge}
Guo, W., Wu, X., Khan, U., and Xing, X.
\newblock Edge: Explaining deep reinforcement learning policies.
\newblock In \emph{Proc. of NeurIPS}, 2021.

\bibitem[Haarnoja et~al.(2018)Haarnoja, Zhou, Abbeel, and Levine]{haarnoja2018soft}
Haarnoja, T., Zhou, A., Abbeel, P., and Levine, S.
\newblock Soft actor-critic: Off-policy maximum entropy deep reinforcement learning with a stochastic actor.
\newblock In \emph{Proc. of ICML}, pp.\  1861--1870, 2018.

\bibitem[Ho \& Ermon(2016)Ho and Ermon]{ho2016generative}
Ho, J. and Ermon, S.
\newblock Generative adversarial imitation learning.
\newblock In \emph{Proc. of NeurIPS}, 2016.

\bibitem[Kakade \& Langford(2002)Kakade and Langford]{kakade2002approximately}
Kakade, S. and Langford, J.
\newblock Approximately optimal approximate reinforcement learning.
\newblock In \emph{Proc. of ICML}, 2002.

\bibitem[Li et~al.(2022)Li, Peng, Feng, Zhang, Xue, and Zhou]{li2022metadrive}
Li, Q., Peng, Z., Feng, L., Zhang, Q., Xue, Z., and Zhou, B.
\newblock Metadrive: Composing diverse driving scenarios for generalizable reinforcement learning.
\newblock \emph{IEEE Transactions on Pattern Analysis and Machine Intelligence}, 2022.

\bibitem[Li et~al.(2023)Li, Zhai, Ma, and Levine]{li2023understanding}
Li, Q., Zhai, Y., Ma, Y., and Levine, S.
\newblock Understanding the complexity gains of single-task rl with a curriculum.
\newblock In \emph{Proc. of ICML}, pp.\  20412--20451, 2023.

\bibitem[Mazoure et~al.(2019)Mazoure, Doan, Durand, Hjelm, and Pineau]{mazoure2019leveraging}
Mazoure, B., Doan, T., Durand, A., Hjelm, R.~D., and Pineau, J.
\newblock Leveraging exploration in off-policy algorithms via normalizing flows.
\newblock In \emph{Proc. of CoRL}, 2019.

\bibitem[Mnih et~al.(2013)Mnih, Kavukcuoglu, Silver, Graves, Antonoglou, Wierstra, and Riedmiller]{mnih2013playing}
Mnih, V., Kavukcuoglu, K., Silver, D., Graves, A., Antonoglou, I., Wierstra, D., and Riedmiller, M.
\newblock Playing atari with deep reinforcement learning.
\newblock \emph{arXiv preprint arXiv:1312.5602}, 2013.

\bibitem[Oh et~al.(2016)Oh, Chockalingam, Lee, et~al.]{oh2016control}
Oh, J., Chockalingam, V., Lee, H., et~al.
\newblock Control of memory, active perception, and action in minecraft.
\newblock In \emph{Proc. of ICML}, 2016.

\bibitem[Oh et~al.(2018)Oh, Guo, Singh, and Lee]{oh2018self}
Oh, J., Guo, Y., Singh, S., and Lee, H.
\newblock Self-imitation learning.
\newblock In \emph{Proc. of ICML}, pp.\  3878--3887, 2018.

\bibitem[Ostrovski et~al.(2017)Ostrovski, Bellemare, Oord, and Munos]{ostrovski2017count}
Ostrovski, G., Bellemare, M.~G., Oord, A., and Munos, R.
\newblock Count-based exploration with neural density models.
\newblock In \emph{Proc. of ICML}, 2017.

\bibitem[Paszke et~al.(2019)Paszke, Gross, Massa, Lerer, Bradbury, Chanan, Killeen, Lin, Gimelshein, Antiga, et~al.]{paszke2019pytorch}
Paszke, A., Gross, S., Massa, F., Lerer, A., Bradbury, J., Chanan, G., Killeen, T., Lin, Z., Gimelshein, N., Antiga, L., et~al.
\newblock Pytorch: An imperative style, high-performance deep learning library.
\newblock In \emph{Proc. of NeurIPS}, 2019.

\bibitem[Peng et~al.(2022)Peng, Li, Liu, and Zhou]{peng2022safe}
Peng, Z., Li, Q., Liu, C., and Zhou, B.
\newblock Safe driving via expert guided policy optimization.
\newblock In \emph{Proc. of CoRL}, 2022.

\bibitem[Puri et~al.(2019)Puri, Verma, Gupta, Kayastha, Deshmukh, Krishnamurthy, and Singh]{puri2019explain}
Puri, N., Verma, S., Gupta, P., Kayastha, D., Deshmukh, S., Krishnamurthy, B., and Singh, S.
\newblock Explain your move: Understanding agent actions using specific and relevant feature attribution.
\newblock In \emph{Proc. of ICLR}, 2019.

\bibitem[Raff et~al.(2017)Raff, Barker, Sylvester, Brandon, Catanzaro, and Nicholas]{raff2017malware}
Raff, E., Barker, J., Sylvester, J., Brandon, R., Catanzaro, B., and Nicholas, C.
\newblock Malware detection by eating a whole exe.
\newblock \emph{arXiv preprint arXiv:1710.09435}, 2017.

\bibitem[Raffin et~al.(2021)Raffin, Hill, Gleave, Kanervisto, Ernestus, and Dormann]{raffin2021stable}
Raffin, A., Hill, A., Gleave, A., Kanervisto, A., Ernestus, M., and Dormann, N.
\newblock Stable-baselines3: Reliable reinforcement learning implementations.
\newblock \emph{Journal of Machine Learning Research}, 2021.

\bibitem[Rashidinejad et~al.(2021)Rashidinejad, Zhu, Ma, Jiao, and Russell]{rashidinejad2021bridging}
Rashidinejad, P., Zhu, B., Ma, C., Jiao, J., and Russell, S.
\newblock Bridging offline reinforcement learning and imitation learning: A tale of pessimism.
\newblock In \emph{Proc. of NeurIPS}, 2021.

\bibitem[Ren et~al.(2019)Ren, Dong, Zhou, Liu, and Peng]{ren2019exploration}
Ren, Z., Dong, K., Zhou, Y., Liu, Q., and Peng, J.
\newblock Exploration via hindsight goal generation.
\newblock In \emph{Proc. of NeurIPS}, 2019.

\bibitem[Schulman et~al.(2017)Schulman, Wolski, Dhariwal, Radford, and Klimov]{schulman2017proximal}
Schulman, J., Wolski, F., Dhariwal, P., Radford, A., and Klimov, O.
\newblock Proximal policy optimization algorithms.
\newblock \emph{arXiv preprint arXiv:1707.06347}, 2017.

\bibitem[Sundararajan et~al.(2017)Sundararajan, Taly, and Yan]{sundararajan2017axiomatic}
Sundararajan, M., Taly, A., and Yan, Q.
\newblock Axiomatic attribution for deep networks.
\newblock In \emph{Proc. of ICML}, 2017.

\bibitem[Todorov et~al.(2012)Todorov, Erez, and Tassa]{todorov2012mujoco}
Todorov, E., Erez, T., and Tassa, Y.
\newblock Mujoco: A physics engine for model-based control.
\newblock In \emph{Proc. of IROS}, 2012.

\bibitem[Uchendu et~al.(2023)Uchendu, Xiao, Lu, Zhu, Yan, Simon, Bennice, Fu, Ma, Jiao, et~al.]{uchendu2023jump}
Uchendu, I., Xiao, T., Lu, Y., Zhu, B., Yan, M., Simon, J., Bennice, M., Fu, C., Ma, C., Jiao, J., et~al.
\newblock Jump-start reinforcement learning.
\newblock In \emph{Proc. of ICML}, 2023.

\bibitem[Van~Waveren et~al.(2022)Van~Waveren, Pek, Tumova, and Leite]{van2022correct}
Van~Waveren, S., Pek, C., Tumova, J., and Leite, I.
\newblock Correct me if i'm wrong: Using non-experts to repair reinforcement learning policies.
\newblock In \emph{Proc. of HRI}, 2022.

\bibitem[Vinyals et~al.(2019)Vinyals, Babuschkin, Czarnecki, Mathieu, Dudzik, Chung, Choi, Powell, Ewalds, Georgiev, et~al.]{vinyals2019grandmaster}
Vinyals, O., Babuschkin, I., Czarnecki, W.~M., Mathieu, M., Dudzik, A., Chung, J., Choi, D.~H., Powell, R., Ewalds, T., Georgiev, P., et~al.
\newblock Grandmaster level in starcraft ii using multi-agent reinforcement learning.
\newblock \emph{Nature}, 2019.

\bibitem[Vyas et~al.(2023)Vyas, Hannay, Bolton, and Burnap]{vyas2023automated}
Vyas, S., Hannay, J., Bolton, A., and Burnap, P.~P.
\newblock Automated cyber defence: A review.
\newblock \emph{arXiv preprint arXiv:2303.04926}, 2023.

\bibitem[Wang et~al.(2023)Wang, Zhang, Hou, and Cheng]{wang2023autonomous}
Wang, X., Zhang, J., Hou, D., and Cheng, Y.
\newblock Autonomous driving based on approximate safe action.
\newblock \emph{IEEE Transactions on Intelligent Transportation Systems}, 2023.

\bibitem[Weng et~al.(2022)Weng, Chen, Yan, You, Duburcq, Zhang, Su, Su, and Zhu]{tianshou2022}
Weng, J., Chen, H., Yan, D., You, K., Duburcq, A., Zhang, M., Su, Y., Su, H., and Zhu, J.
\newblock Tianshou: A highly modularized deep reinforcement learning library.
\newblock \emph{Journal of Machine Learning Research}, 2022.

\bibitem[Xie et~al.(2021)Xie, Jiang, Wang, Xiong, and Bai]{xie2021policy}
Xie, T., Jiang, N., Wang, H., Xiong, C., and Bai, Y.
\newblock Policy finetuning: Bridging sample-efficient offline and online reinforcement learning.
\newblock In \emph{Proc. of NeurIPS}, 2021.

\bibitem[Yang et~al.(2017)Yang, Kong, Xie, and Gunter]{yang2017malware}
Yang, W., Kong, D., Xie, T., and Gunter, C.~A.
\newblock Malware detection in adversarial settings: Exploiting feature evolutions and confusions in android apps.
\newblock In \emph{Proc. of ACSAC}, 2017.

\bibitem[Yu et~al.(2023)Yu, Guo, Qin, Wang, Wang, and Xing]{yu2023airs}
Yu, J., Guo, W., Qin, Q., Wang, G., Wang, T., and Xing, X.
\newblock Airs: Explanation for deep reinforcement learning based security applications.
\newblock In \emph{Proc. of USENIX Security}, 2023.

\bibitem[Zhang et~al.(2021)Zhang, Liniger, Dai, Yu, and Van~Gool]{zhang2021end}
Zhang, Z., Liniger, A., Dai, D., Yu, F., and Van~Gool, L.
\newblock End-to-end urban driving by imitating a reinforcement learning coach.
\newblock In \emph{Proc. of ICCV}, 2021.

\end{thebibliography}
\bibliographystyle{icml2024}

\newpage
\appendix
\onecolumn
\section{Proof of \sref{Theorem}{theorem:upper_bound}}
\label{appendix:proof_theorem33}

Based on the Performance Difference Lemma~\cite{kakade2002approximately}, we have
\begin{equation}
     \eta(\bar{\pi}) - \eta(\pi) = \frac{1}{1-\gamma} \mathbb{E}_{s \sim d_\rho^{\bar{\pi}}} \mathbb{E}_{a \sim \bar{\pi}(\cdot|s)}A^{\pi}\left(s, a\right) ,
\end{equation}
\noindent where $\pi$ is the policy of the target agent, $\bar{\pi}$ is the perturbed policy, $\rho$ is the initial distribution, and $\gamma$ is the discount rate.

Note that the perturbed policy $\bar{\pi}$ is a mixture of the target agent's policy $\pi$ and a random policy $\pi^r$ (\ie $\bar{\pi}(\cdot|s) = \tilde{\pi}(a^e = 0|s) \pi(\cdot|s) + \tilde{\pi}(a^e = 1|s)\pi^r(\cdot|s)$). Denote the probability of the mask network outputting 0 at state $s$ as $\tilde{\pi}(a^e = 0|s) = \xi(s)$ and the probability of the mask network outputting 1 at state $s$ as $\tilde{\pi}(a^e = 1|s) = 1 - \xi(s)$ Given the fact that $A^{\pi}\left(s, \pi(\cdot|s)\right) = \mathbb{E}_{a \sim \pi(s)} A^{\pi} (s,a) = 0$, we have 

\begin{equation}
\begin{split}
     \eta(\bar{\pi}) - \eta(\pi) 
     &=\frac{1}{1-\gamma} \mathbb{E}_{s \sim d_\rho^{\bar{\pi}}} \mathbb{E}_{a \sim \bar{\pi}(\cdot|s)}A^{\pi}\left(s, a\right)\\
     &= \frac{1}{1-\gamma} \mathbb{E}_{s \sim d_\rho^{\bar{\pi}}} \sum_{a}  \bar{\pi}(a|s) A^{\pi}\left(s, a\right)\\
     &= \frac{1}{1-\gamma} \mathbb{E}_{s \sim d_\rho^{\bar{\pi}}} \sum_{a}  \xi(s) \pi(a|s) A^{\pi}\left(s, a\right) + \frac{1}{1-\gamma} \mathbb{E}_{s \sim d_\rho^{\bar{\pi}}} \sum_{a}  (1-\xi(s)) \pi^r(a|s) A^{\pi}\left(s, a\right)\\
     &= \frac{1}{1-\gamma} \mathbb{E}_{s \sim d_\rho^{\bar{\pi}}} \xi (s) \mathbb{E}_{a \sim   \pi(\cdot|s)} A^{\pi}\left(s, a\right) + \frac{1}{1-\gamma} \mathbb{E}_{s \sim d_\rho^{\bar{\pi}}} (1-\xi (s))\mathbb{E}_{a \sim  \pi^r(\cdot|s)} A^{\pi}\left(s, a\right)\\
     &= \frac{1}{1-\gamma} \mathbb{E}_{s \sim d_\rho^{\bar{\pi}}} (1-\xi (s))\mathbb{E}_{a \sim \pi^r(\cdot|s)} A^{\pi}\left(s, a\right) \leq 0.\\
\end{split}
\end{equation}

Therefore, we show that $\eta(\bar{\pi})$ is upper bounded by $\eta(\pi)$ given \sref{Assumption}{assumption:random}.

\section{Theoretical Guarantee}
\subsection{Validation of \sref{Assumption}{assumption:better} in a 2-state MDP}
\label{appendix:2state}
In a 2-state MDP, we have two different states, namely, $s_A$ and $s_B$. The state distribution of any policy $\pi$ follows $d_{\rho}^{\pi}(s_A) + d_{\rho}^{\pi}(s_B) = 1$. As such, the set of feasible state marginal distribution can be described by a line $[(0, 1), (1, 0)]$ in $\mathbb{R}^2$. Let's denote vector $\mathbf{s} = [s_A, s_B]$. The expected total reward of a policy $\pi$ can be represented as $\eta(\pi) = <d_{\rho}^{\pi} (\mathbf{s}), R(\mathbf{s})>$~\cite{eysenbach2021information}, where $R(\mathbf{s})=[R(s_A), R(s_B)]$. \autoref{fig:vis_2mdp} shows the area of achievable state distribution via the initial state distribution $\rho$ (highlighted in orange).  

It should be noted that not all the points in the line $[(0, 1), (1, 0)]$ corresponded to a valid Markovian policy. However, for any convex combination of valid state occupancy measures, there exists a Markovian policy that has this state occupancy measure. As such, the policy search occurs within a convex polytope, essentially a segment (\ie, marked in orange) along this line.  In \autoref{fig:vis_2mdp}, we visualize $R(\bf{s})$ as vectors starting at the origin. Since $V^{\hat{\pi}}(\rho) \geq V^{\pi}(\rho)$, We mark $d_{\rho}^{\hat{\pi}}(\mathbf{s})$ closer to $R(\mathbf{s})$ (\ie the inner product between $d_{\rho}^{\hat{\pi}}(\mathbf{s})$ and $R(\bf{s})$ and is larger than $d_{\rho}^{\pi}(\mathbf{s})$ and $R(\bf{s})$). The following theorem explains how we determine the location of the location of $d_{\rho}^{\pi^*}(\mathbf{s})$ in \autoref{fig:vis_2mdp}.

\begin{theorem}[Fact 1~\cite{eysenbach2021information}]
For every state-dependent reward function, among the set of policies that maximize that reward function is one that lies at a vertex of the state marginal polytope.
\label{theorem:vertex}
\end{theorem}

According to \sref{Theorem}{theorem:vertex},  $d_{\rho}^{\pi^*}(\mathbf{s})$ located at either vertex in the orange segment. Since $\pi^*$ is the optimal policy, it lies at the vertex that has the larger inner product within $R(s)$. Once the position of  $d_{\rho}^{\pi^*}(\mathbf{s})$ is determined, we can easily find $\left\|\frac{d_\rho^{\pi^*}(s)}{d_{\rho}^{\hat{\pi}} (s)}\right\|_{\infty} \leq \left\|\frac{d_\rho^{\pi^*}(s)}{d_{\rho}^{\pi} (s)}\right\|_{\infty}$ based on \autoref{fig:vis_2mdp}.

\begin{figure}[h]
\centering
\includegraphics[width=0.3\textwidth]{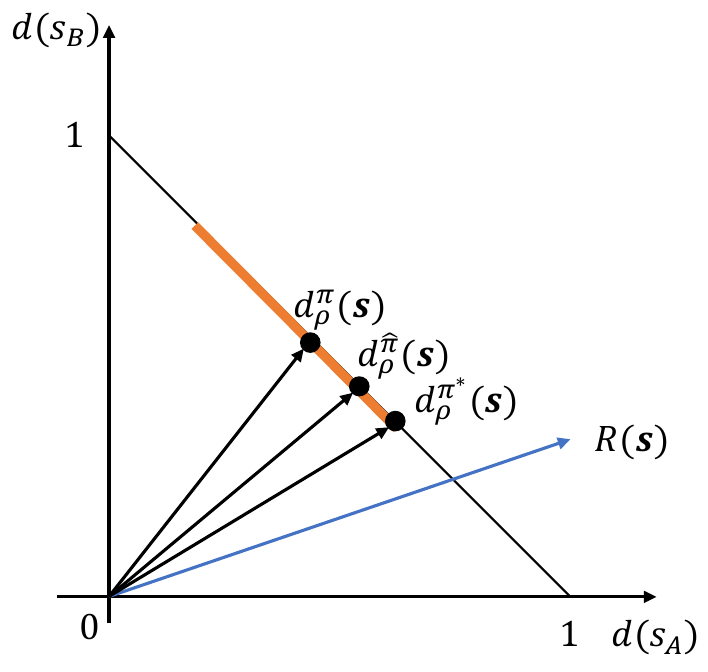}
\caption{Visualization of state occupancy measures with respect to different policies and the reward function in a 2-state MDP.}
\label{fig:vis_2mdp}
\vspace{-10pt}
\end{figure}

\subsection{Proof of \sref{Lemma}{lemma:sampling}}
\label{appendix:lemma_q1}
\begin{proof}
Since our explanation method provides the importance of each state, we could view the sampling based on the state's importance as a reweighting of the state occupancy measure. Mathematically, it can be expressed as $d_{\rho}^{\hat{\pi}}(s) = d_{\rho}^\pi(s) w(s)$, where $\hat{\pi}$ is the equivalent policy of reweighting the original policy $\pi$ and $w(s)$ is the weight provided by the mask network. Although the mask network takes the current input state as input, it implicitly considers the current action as well, as detailed by StateMask~\cite{cheng2023statemask}. Consequently, a more accurate formulation is  $d_{\rho}^{\hat{\pi}}(s, a) = d_{\rho}^\pi(s, a) w(s, a)$, where $w(s, a)$ represent the weight assigned by mask network.


Recall that our proposed explanation method is to randomize actions at non-critical steps, which essentially considers the value of $Q_{diff} = Q^{\pi}(s,a) - \mathbb{E}_{a' \in \mathcal{A}} [Q^{\pi}(s,a')]$. In fact, a larger $Q_{diff}$ indicates current time step is more critical to the agent's final reward. Our mask network approximates the value of $Q_{diff}$ via the deep neural network to determine the importance of each step, which implies $w(s, a) \propto Q_{diff} \propto Q^{\pi}(s,a)$. 

Next, we aim to prove that our MaskNet-based sampling approach is equivalent to sampling from a better policy $\hat{\pi}$. 

First, the equivalent policy $\hat{\pi}$ after reweighting can be expressed as
\begin{equation}
\small
\hat{\pi}(a|s) = \frac{d_{\rho}^{\hat{\pi}}(s,a)}{d_{\rho}^{\hat{\pi}}(s)}    
= \frac{d_{\rho}^\pi(s,a) w(s,a)}{d_{\rho}^{\hat{\pi}}(s)}
= w(s,a)\pi(a|s)\frac{d_{\rho}^\pi(s)}{d_{\rho}^{\hat{\pi}}(s)}.
\end{equation}

Further , we would like to show that if $w(s,a) = f(Q^{\pi}(s,a))$ where $f(\cdot)$ is a monotonic increasing function, $\hat{\pi}$ is uniformly as good as,or better than $\pi$, \ie $V^{\hat{\pi}}(s) \geq V^{\pi}(s)$.

\begin{proposition}
Suppose two policies $\hat{\pi}$ and $\pi$ satisfy $g\left(\hat{\pi}(a | s)\right)=g(\pi(a | s))+h\left(s, Q^\pi(s, a)\right)$ where $g(\cdot)$ is a monotonically increasing function, and $h(s,\cdot)$ is monotonically increasing for any fixed $s$ . Then we have $V^{\hat{\pi}}(s) \geq V^{\pi}(s)$, $\forall s \in \mathcal{S}$.
\label{proposition:Q re-weight}
\end{proposition}

\begin{proof}
For a given $s$, we partition the action set $\mathcal{A}$ into two subsets $\mathcal{A}_1$ and $\mathcal{A}_2$.
$$
\mathcal{A}_1 \triangleq\left\{a \in \mathcal{A} | \hat{\pi}(a | s) \geqslant \pi(a | s)\right\}.
$$
$$
\mathcal{A}_2 \triangleq\left\{a \in \mathcal{A} | \hat{\pi}(a | s) < \pi(a | s)\right\}.
$$

Thus, $\forall a_1 \in \mathcal{A}_1$, $\forall a_2 \in \mathcal{A}_2$, we have
\begin{equation}
\begin{split}
h\left(s, Q^\pi(s, a_1)\right) &= g\left(\hat{\pi}(a_1 | s)\right)- g(\pi(a_1 | s)) \\
&\geq 0 \\
&\geq g\left(\hat{\pi}(a_2 | s)\right)- g(\pi(a_2 | s))\\
& = h\left(s, Q^\pi(s, a_2)\right).
\end{split}
\end{equation}

Let $h\left(s, Q^\pi(s, a)\right) = Q^\pi(s, a)$. We can get  $Q^\pi(s, a_1) \geq Q^\pi(s, a_2)$ which means we can always find $q(s) \in \mathbb{R}$ such that $Q^\pi(s, a_1) \geq q(s) \geq Q^\pi(s, a_2)$, $\forall a_1 \in \mathcal{A}_1$, $\forall a_2 \in \mathcal{A}_2$. Thus,
\begin{equation}
\begin{split}
& \sum_{a \in \mathcal{A}} \hat{\pi}(a | s) Q^\pi(s, a)-\sum_{a \in A} \pi(a | s) Q^\pi(s, a) \\
= & \sum_{a_1 \in \mathcal{A}_1}\left(\hat{\pi}\left(a_1 | s\right)-\pi(a_1 | s)\right) Q^\pi\left(s, a_1\right)  +\sum_{a_2 \in \mathcal{A}_2}\left(\hat{\pi}\left(a_2 | s\right)-\pi(a_2 | s)\right) Q^\pi ( s, a_2)\\
\geq &  \sum_{a_1 \in \mathcal{A}_1}\left(\hat{\pi}\left(a_1 | s\right)-\pi\left(a_1 | s\right)\right) q(s) +  \sum_{a_2 \in \mathcal{A}_2}\left(\hat{\pi}\left(a_2 | s\right)-\pi\left(a_2 | s\right)\right) q(s)\\
= & \hspace{0.5em}  q(s) \sum_{a \in \mathcal{A}}\left(\pi'(a | s)-\pi(a| s)\right)\\
= & \hspace{0.5em}  0.\\
\end{split}
\end{equation}

Let $V_0 (s) = V^{\pi}(s)$. And we denote the value function of following $\hat{\pi}$ for $l$ steps then following $\pi$ as $V_l(s)=E_{a \sim \hat{\pi}(.| s)}\left[E_{s^{\prime}, r | s, a}\left(r+\gamma V_{l-1}\left(s^{\prime}\right)\right)\right]$ if $l \geq 1$.

First, we observe that 
\begin{equation}
\begin{aligned}
V_1(s) & =E_{a \sim \hat{\pi}(.| s)}\left[E_{s^{\prime}, r | s, a}\left(r+\gamma V^{\pi}(s^{\prime})\right)]\right. \\
& =\sum_{a \in \mathcal{A}} \hat{\pi}(a| s) Q^\pi(s, a) \\
& \geqslant \sum_{a \in \mathcal{A}} \pi(a | s) Q^\pi(s, a)\\
&=V_0(s).
\end{aligned}
\end{equation}

By induction, we assume $V_l (s) \geq V_{l-1}(s)$. Given that
$$
V_{l+1}(s)=E_{a \sim \hat{\pi}}\left[E_{s^{\prime}, r | s, a}\left(r+V_l\left(s^{\prime}\right)\right)\right],
$$
$$
V_l(s)=E_{a \sim \hat{\pi}}\left[E_{s^{\prime}, r | s, a}\left(r+V_{l-1}\left(s^{\prime}\right)\right)\right],
$$
we have $V_{l+1}(s) \geq V_l(s)$.

Therefore, we can conclude that $V_{l+1}(s) \geq V_l(s)$, $\forall l\geq 0$. We have $V_{\infty} (s) \geq V_0 (s)$ which is $V^{\hat{\pi}} (s) \geq V^{\pi} (s)$.
\end{proof}

Based on the \sref{Proposition}{proposition:Q re-weight}, if we choose $g$ as a logarithmic function and $h=\log(w(s,a)) + \log(d_{\rho}^\pi(s)) - \log(d_{\rho}^{\hat{\pi}}(s))$, we can easily verify that our MaskNet-based sampling approach 
 is equivalent to sampling from a better policy $\hat{\pi}$.

\end{proof}
\subsection{Proof of \sref{Theorem}{theorem:performance_diff}}
\label{appendix:proof_theorem_1}
\begin{proof}
Given the fact that the refined policy $\pi'$ is converged, (\ie the local one-step improvement is small $\mathbb{E}_{s \sim d_{\mu}^{\pi'}}\left[\max _a A^{\pi'}(s, a)\right] < \epsilon$), we have
\vspace{-10pt}
\begin{equation}
\begin{split}
\epsilon &> \sum_{s \in \mathcal{S}} d_{\mu}^{\pi'} (s)\left[\max _a A^{\pi'}(s, a)\right]\\
&\geq \min_{s} \left(\frac{d_\mu^{\pi'}(s)}{d_\rho^{\pi^*}(s)}\right) \sum_{s} d_\rho^{\pi^*}(s) \max _a A^{\pi'}(s, a)\\
&\geq\left\|\frac{d^{\pi^*}_{\rho}}{d_\mu^{\pi'}}\right\|_{\infty}^{-1} \sum_{s, a} d_\rho^{\pi^*}(s) \pi^*(a | s) A^{\pi'}(s, a).\\
\end{split}
\end{equation}

Based on the Performance Difference Lemma~\cite{kakade2002approximately},
for two policies $\pi^*, \pi^{\prime}$ and a state distribution $\rho$, the performance difference is bounded by
\begin{equation}
\begin{aligned}
V^{\pi^*}\left(\rho\right)-V^{\pi^{\prime}}\left(\rho\right) = \frac{1}{1-\gamma} E_{s \sim d_{\rho}^{\pi*}} E_{a \sim \pi^*(.| s)}\left[A^{\pi^{\prime}}(s, a)\right].
\end{aligned}    
\end{equation}

Then we have
\begin{equation}
\begin{split}
\epsilon & > (1-\gamma)\left\|\frac{d^{\pi^*}_{\rho}}{d_\mu^{\pi'}}\right\|_{\infty}^{-1} \left(V^\pi\left(\rho\right)-V^{\pi^{\prime}}\left(\rho\right)\right).
\end{split}
\end{equation}

Therefore, we have 
\begin{equation}
\small
V^{\pi^*}(\rho)-V^{\pi'}(\rho) 
\leq  \frac{\epsilon}{1-\gamma} \left\|\frac{d_\rho^{\pi^*}}{d_\mu^{\pi'}}\right\|_{\infty}.    
\end{equation}

Due to $d_\mu^{\pi'} (s) \geq (1-\gamma) \mu(s)$, we further obtain 
\begin{equation}
\small
V^{\pi^*}(\rho)-V^{\pi'}(\rho) 
\leq \frac{\epsilon}{(1-\gamma)^2} \left\|\frac{d_\rho^{\pi^*}}{\mu}\right\|_{\infty}.
\end{equation}
Since $\mu(s) = \beta d_{\rho}^{\hat{\pi}}(s) + (1 - \beta)\rho(s) \geq \beta d_{\rho}^{\hat{\pi}}(s)$, we have
\begin{equation}
\small
V^{\pi^*}(\rho)-V^{\pi'}(\rho) 
\leq \frac{\epsilon}{(1-\gamma)^2} \left\|\frac{d_\rho^{\pi^*}}{\beta d_{\rho}^{\hat{\pi}}}\right\|_{\infty}.
\end{equation}
In our case, $\beta$ is a constant (\ie a hyper-parameter), thus we could derive that
\begin{equation}
\small
V^{\pi^*}(\rho)-V^{\pi'}(\rho) 
\leq \mathcal{O}(\frac{\epsilon}{(1-\gamma)^2} \left\|\frac{d_\rho^{\pi^*}}{d_{\rho}^{\hat{\pi}}}\right\|_{\infty}),
\end{equation}
\noindent which completes the proof.
\end{proof}

\subsection{Analysis of \sref{Claim}{claim:tight_bound}}
\label{appendix:claim_1}

Recall that \sref{Lemma}{lemma:sampling} indicates that our MaskNet-based sampling approach is equivalent to sampling states from a better policy $\hat{\pi}$ compared with a random explanation sampling from the policy $\pi$, \ie $\eta(\hat{\pi}) \geq \eta(\pi)$. Let us denote the new initial distribution using our MaskNet-based sampling approach as $\mu$. By \sref{Assumption}{assumption:better}, we have $\left\|\frac{d_\rho^{\pi^*}}{d_{\rho}^{\hat{\pi}}}\right\|_{\infty} \leq \left\|\frac{d_\rho^{\pi^*}}{d_{\rho}^{\pi}}\right\|_{\infty}$. Using our explanation method introduces a smaller distribution mismatch coefficient than using a random explanation method. Therefore, we claim that using our explanation method, we are able to form a better initial distribution $\mu$ and tighten the upper bound in \sref{Theorem}{theorem:performance_diff}, \ie enhancing the agent's performance after refining.

\section{Details of Evaluation}
\label{appendix:evaluation}

\subsection{Implementation Details}

\para{Implementation of Our Method.} We implement the proposed method using PyTorch~\cite{paszke2019pytorch}. We implement our method in four selected MuJoCo games based on Stable-Baselines3~\cite{raffin2021stable}. We train the agents on a server with 8 NVIDIA A100 GPUs for all the learning algorithms. For all our experiments, if not otherwise stated, we use a set of default hyper-parameters for $p$, $\lambda$, and $\alpha$ (listed in \sref{Appendix}{appendix:exp-add}). 

We implement the environment reset function similar to \citet{ecoffet2019go} to restore the environment to selected critical states. This method is feasible in our case, as we operate within simulator-based environments. However, in the real world, it may not be always possible to return to a certain state with the same sequences of actions due to the stochastic nature of state transition. It's important to note that our framework is designed to be versatile and is indeed compatible with a goal/state-conditioned policy approach such as \citet{ecoffet2021first}. Given a trajectory with an identified most important state, we can select the most important state as the final goal and select the en-route intermediate states as sub-goals. Then we can train an agent to reach the final goal by augmenting each state with the next goal and giving a goal-conditioned reward once the next goal is reached until all goals are achieved.

\para{Implementation of Baseline Methods.} Regarding baseline approaches, we use the code released by the authors or implement our own version if the authors don't release the code. Specifically, as for StateMask, we use their released open-sourced code from~\url{https://github.com/nuwuxian/RL-state_mask}. Regarding Jump-Start Reinforcement Learning, we use the implementation from~\url{https://github.com/steventango/jumpstart-rl}.

\subsection{Extra Introduction to Applications}
\label{appendix:app_description}
\para{Hopper.}
Hopper game is a captivating two-dimensional challenge featuring a one-legged figure comprising a torso, thigh, leg, and a single supporting foot~\cite{erez2012infinite}. The objective is to propel the Hopper forward through strategic hops by applying torques to the three hinges connecting its body parts. Observations include positional values followed by velocities of each body part, and the action space involves applying torques within a three-dimensional action space. Under the {\tt dense} reward setting, the reward system combines healthy reward, forward reward, and control cost. Under the {\tt sparse} reward setting~\cite{mazoure2019leveraging}, the reward informs the x position of the hopper only if $x>0.6$ in our experiments. The episode concludes if the Hopper becomes unhealthy. We use ``Hopper-v3'' in our experiments.

\para{Walker2d.} Walker2d is a dynamic two-dimensional challenge featuring a two-legged figure with a torso, thighs, legs, and feet. The goal is to coordinate both sets of lower limbs to move in the forward direction by applying torques to the six hinges connecting these body parts. The action space involves six dimensions, allowing exert torques at the hinge joints for precise control. Observations encompass positional values and velocities of body parts, with the former preceding the latter. Under the {\tt dense} reward setting, the reward system combines a healthy reward bonus, forward reward, and control cost. Under the {\tt sparse} reward setting~\cite{mazoure2019leveraging}, the reward informs the x position of the hopper only if $x>0.6$ in our experiments. The episode concludes if the walker is deemed unhealthy. We use ``Walker2d-v3'' in our experiments and normalize the observation when training the DRL agent.

\para{Reacher.} Reacher is an engaging two-jointed robot arm game where the objective is to skillfully maneuver the robot's end effector, known as the fingertip, towards a randomly spawned target. The action space involves applying torques at the hinge joints. Observations include the cosine and sine of the angles of the two arms, the coordinates of the target, angular velocities of the arms, and the vector between the target and the fingertip. It is worth noting that there is no sparse reward implementation of Reacher-v2 in~\citet{mazoure2019leveraging}. The reward system comprises two components: ``reward distance'' indicating the proximity of the fingertip to the target, and ``reward control'' penalizing excessive actions with a negative squared Euclidean norm. The total reward is the sum of these components, and an episode concludes either after 50 timesteps with a new random target or if any state space value becomes non-finite. We use ``Reacher-v2'' in our experiments.

\para{HalfCheetah.} 
HalfCheetah is an exhilarating 2D robot game where players control a 9-link cheetah with 8 joints, aiming to propel it forward with applied torques for maximum speed. The action space contains six dimensions, that enable strategic movement. Observations include positional values and velocities of body parts. Under the {\tt dense} reward setting, the reward system balances positive ``forward reward'' for forward motion with ``control cost'' penalties for excessive actions. Under the {\tt sparse} reward setting~\cite{mazoure2019leveraging}, the reward informs the x position of the hopper only if $x>5$ in our experiments. Episodes conclude after 1000 timesteps, offering a finite yet thrilling gameplay experience. We use ``HalfCheetah-v3'' in our experiments and normalize the observation when training the DRL agent.

\para{Selfish Mining.}  
Selfish mining is a security vulnerability in blockchain protocols, identified by \citet{eyal2018majority}. 
When a miner holds a certain amount of computing power, they can withhold their freshly minted blocks from the public blockchain, thereby initiating a fork that is subsequently mined ahead of the official public blockchain. With this advantage, the miner can introduce this fork into the network, overwriting the original blockchain and obtaining more revenue. 

To find the optimal selfish mining strategies, \citet{bar2022werlman} proposed a deep reinforcement learning model to generate a mining policy. The policy takes the current chain state as the input and chooses from the three pre-determined actions, \ie adopting, revealing, and mining. With this policy network, the miner can obtain more mining rewards compared to using heuristics-based strategies. 

We train a PPO agent in the blockchain model developed by Bar-Zur~\etal~\cite{githubGitHubRoibarzurptoselfishmining}. The network architecture of the PPO agent is a 4-layer Multi-Layer Perceptron (MLP) with a hidden size of 128, 128, 128, and 128 in each layer. We adopt a similar network structure for training our mask network. The whale transaction has a fee of 10 with the occurring probability of 0.01 while other normal transactions have a fee of 1. The agent will receive a positive reward if his block is accepted and will be penalized if his action is determined to be unsuccessful, \eg revealing a private chain.

In our selfish mining task~\cite{bar2022werlman}, three distinct actions are defined as follows:

\emph{Adopt $l$}: The miner chooses to adopt the first $l$ blocks in the public chain while disregarding their private chain. Following this, the miner will continue their mining efforts, commencing from the last adopted block.

\emph{Reveal $l$}: This action becomes legal when the miner's private chain attains a length of at least $l$. The consequence of this action may result in either the creation of an active fork in the public chain or the overriding of the public chain.

\emph{Mine}: This action simply involves continuing with the mining process. Once executed, a new block is mined and subsequently added to either the private chain of the rational miner or to the public chain, contingent on which entity successfully mined the block.

\para{CAGE Challenge 2.} 
To inspire new methods for automating cyber defense, the Technical Cooperation Program (TTCP) launched the Autonomous Cyber Defence Challenge (CAGE Challenge) to produce AI-based blue teams for instantaneous response against cyber attacks~\cite{cage_challenge_2_announcement}. The goal is to create a DRL blue agent to protect a network against a red agent. The action space of the blue agent includes monitoring, analyzing, decoyApache, decoyFemitter, decoyHarakaSMPT, decoySmss, decoySSHD, decoySvchost, decoyTomcat, removing, and restoring. Note that the blue agent can receive a {\em negative reward} when the red agent gets admin access to the system (and continues to receive negative rewards as the red agent maintains the admin access). We use CAGE challenge 2 for our evaluation. 

We choose the champion scheme proposed by Cardiff University~\cite{githubGitHubJohncardiffcyborgcage2} in CAGE challenge 2~\cite{githubGitHubCagechallengecagechallenge2}. The target agent is a PPO-based blue agent to defend a network against the red agent ``B-line''. The trail has three different lengths, \ie 30, 50, and 100. The final reward is the sum of the average rewards of these three different lengths.

The action set of the blue agent is defined as follows.

\emph{Monitor}: The blue agent automatically collects the information about flagged malicious activity on the system and reports network connections and associated processes that are identified as malicious.

\emph{Analyze}: The blue agent analyzes the information on files associated with recent alerts including signature and entropy.

\emph{DecoyApache, DecoyFemitter, DecoyHarakaSMPT, DecoySmss, DecoySSHD, DecoySvchost, DecoyTomcat}: The blue agent sets up the corresponding decoy service on a specified host. An alert will be raised if the red agent accesses the decoy service.

\emph{Remove}: The blue agent attempts to remove red from a host by destroying malicious processes, files, and services. 

\emph{Restore}: The blue agent restores a system to a known good state. Since it significantly impacts the system's availability, a reward penalty of -1 will be added when executing this action.

\para{Autonomous Driving.} 
Deep reinforcement learning has been applied in autonomous driving to enhance driving safety. One representative driving simulator is MetaDrive~\cite{li2022metadrive}. A DRL agent is trained to guide a vehicle safely and efficiently to travel to its destination. MetaDrive converts the Birds Eye View (BEV) of the road conditions and the sensor information such as the vehicle's steering, direction, velocity, and relative distance to traffic lanes into a vector representation of the current state. The policy network takes this state vector as input and yields driving actions, including accelerating, braking, and steering commands. MetaDrive employs a set of reward functions to shape the learning process. For instance, penalties are assigned when the agent collides with other vehicles or drives out of the road boundary. To promote smooth and efficient driving, MetaDrive also incorporates rewards to encourage forward motion and the maintenance of an appropriate speed.

We select the ``Macro-v1'' environment powered by the MetaDrive simulator~\cite{li2022metadrive}. The goal of the agent is to learn a deep policy to successfully cross the car flow and reach the destination. We train the target agent and our mask network by the PPO algorithm following the implementation of DI-drive~\cite{didrive}. The environment receives normalized action to control the target agent $\mathbf{a} = [a_1, a_2] \in [-1,1]^2$. The action vector $\mathbf{a}$ will then be converted to the steering (degree), acceleration (hp), and brake signal (hp).

\para{Malware Mutation.} 
DRL has been used to assess the robustness of ML-based malware detectors. For example, \citet{anderson2018learning} propose a DRL-based approach to attack malware detectors for portable executable (PE) files. We use the ``Malconv'' gym environment \citet{raff2017malware} implemented in~\cite{githubGitHubBfilarmalware_rl}  for our experiments. We train a DRL agent based on Tianshou framework~\cite{tianshou2022}. The input of the DRL agent is a feature vector of the target malware and outputs the corresponding action to guide the malware manipulation. We present the action set of the MalConv gym environment in \autoref{table:malconv_action} for ease of comprehension in the case study section. A big reward of 10 is provided when evading detection.

The reward mechanism of the ``Malconv'' environment is as follows. Initially, the malware detection model will provide a score $sc_0$ of the current malware. If $sc_0$ is lower than some threshold, the malware has already evaded the detection. Otherwise, the DRL agent will take some mutation actions to bypass the detection. At step $t$, after executing the agent's action, the malware detection model will provide a new score $sc_t$. If $sc_t$ is lower than some threshold, the mutation is successful and a big reward of 10 will be given. Otherwise, the reward will be $sc_0 - sc_t$. The maximum allowed number of steps is 10.

\begin{table}[t]
\centering
\caption{Action set of the MalConv gym environment.
}
\small
\begin{tabular}{c|c}
\Xhline{1pt}
Action index & Action meaning                      \\ \Xhline{1pt}
0            & ``modify\_machine\_type''            \\ \hline
1            & ``pad\_overlay''                      \\ \hline
2            & ``append\_benign\_data\_overlay''     \\ \hline
3            & ``append\_benign\_binary\_overlay''   \\ \hline
4            & ``add\_bytes\_to\_section\_cave''     \\ \hline
5            & ``add\_section\_strings''             \\ \hline
6            & ``add\_section\_benign\_data''        \\ \hline
7            & ``add\_strings\_to\_overlay''         \\ \hline
8            & ``add\_imports''                      \\ \hline
9            & ``rename\_section''                   \\ \hline
10           & ``remove\_debug''                     \\ \hline
11           & ``modify\_optional\_header''          \\ \hline
12           & ``modify\_timestamp''                 \\ \hline
13           & ``break\_optional\_header\_checksum'' \\ \hline
14           & ``upx\_unpack''                       \\ \hline
15           & ``upx\_pack''                         \\ \Xhline{1pt}
\end{tabular}
\label{table:malconv_action}
\end{table}

\subsection{Additional Experiment Results}
\label{appendix:exp-add}

\para{Hyper-parameter Choices in Experiment I-V.} \autoref{table:transposed-hyper-parameters} summarizes our hyper-parameter choices in Experiment I-V. For all applications, we choose the coefficient of the intrinsic reward for training the mask network $\alpha$ as 0.01. The hyper-parameters $p$ and $\lambda$ for our refining method vary by application.

\begin{table}[t]
\caption{Hyper-parameter choices in Experiment I-V. ``Selfish'' represents Selfish Mining. ``Cage'' represents Cage Challenge 2. ``Auto'' represents Autonomous Driving. ``Malware'' represents Malware Mutation.}
\centering
\small
\begin{tabular}{c?c|c|c|c|c|c|c|c}
\Xhline{1pt}
Hyper-parameter  & Hopper & Walker2d  & Reacher & HalfCheetah           & Selfish & Cage & Auto & Malware \\ \Xhline{1pt}
$p$   & 0.25 & 0.25 & 0.50 & 0.50            & 0.25           & 0.50              & 0.25          & 0.50              \\ \hline
$\lambda$  & 0.001 & 0.01 & 0.001 &  0.01      & 0.001          & 0.01              & 0.01          & 0.01              \\ \hline
$\alpha$  & 0.0001 & 0.0001 &0.0001  &  0.0001        & 0.0001           & 0.0001              & 0.0001          & 0.0001              \\ \Xhline{1pt}
\end{tabular}
\label{table:transposed-hyper-parameters}
\end{table}

\para{Fidelity Scores in Experiment I.} \autoref{fig:fid_score}  shows the fidelity score comparison across all explanation methods. We have three key observations.
First, We observe that our explanation method has similar fidelity scores with StateMask across all applications, empirically indicating the equivalence of our explanation method with StateMask.
Second, we observe that our explanation method and StateMask have higher fidelity scores than random explanation across all applications, indicating that the mask network provides more faithful explanations for the target agents. 

\begin{figure}[t]
    \centering
    \includegraphics[width=\textwidth]{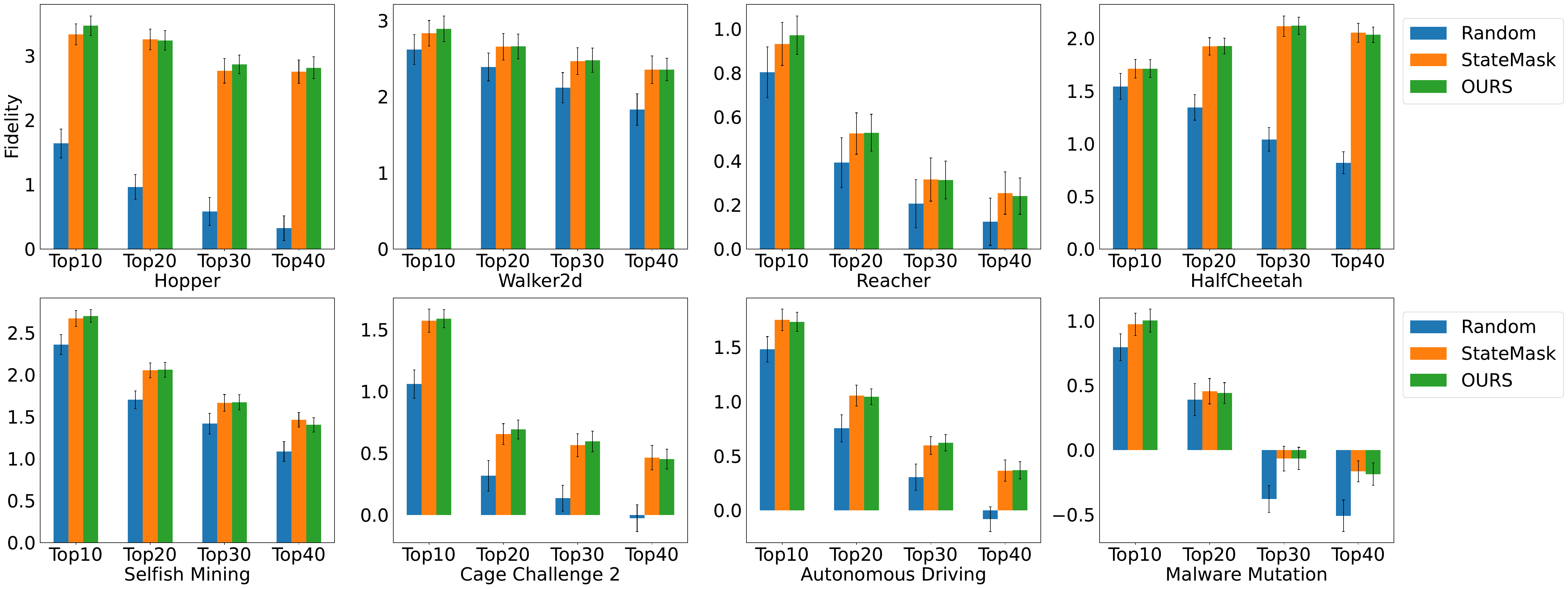}
    \caption{Fidelity scores for explanation generated by baseline methods and our proposed explanation method. Note that a higher score implies higher fidelity.}
    \label{fig:fid_score}
\end{figure}

\para{Efficiency Comparison in Experiment II.} 
\autoref{table:efficiency} reports the efficiency evaluation results when training a mask network using StateMask and our method. We observe that it takes 16.8\% less time on average to train a mask network using our method than using StateMask, which shows the advantage of our method with respect to efficiency.

\begin{table}[t]
\centering
\small
\caption{Efficiency comparison when training the mask network. We report the number of seconds when training the mask using a fixed number of samples. ``Selfish'' represents Selfish Mining. ``Cage'' represents Cage Challenge 2. ``Auto'' represents Autonomous Driving. ``Malware'' represents Malware Mutation.}
\begin{tabular}{c?c|c|c|c|c|c|c|c}
\Xhline{1pt}
Applications    & Hopper & Walker2d & Reacher & HalfCheetah & Selfish & Cage & Auto & Malware \\ \Xhline{1pt}
Num. of samples & $3 \times 10^5$       &    $3 \times 10^5$      & $3 \times 10^5$        &  $3 \times 10^5$           & $1.5 \times 10^6$        & $1 \times 10^7$     & 2443260     &  32349       \\ \Xhline{1pt}
StateMask       & 15393       &  2240        &  8571       &  1579           & 9520         & 79382      & 109802     &  50775       \\ \hline
Ours            & \textbf{12426}       & \textbf{1899}         & \textbf{7033}        &   \textbf{1317}          &  \textbf{8360}       & \textbf{65400}     & \textbf{88761}     & \textbf{41340}        \\ \Xhline{1pt}
\end{tabular}
\vspace{-10pt}
\label{table:efficiency}
\end{table}

\para{Comparison with Self-Imitation Learning.} We compare {\tt RICE} against the self-imitation learning (SIL) approach~\cite{oh2018self} across four MuJoCo games. We present the results presented in~\autoref{tab:self_imitation}. These experiment results demonstrate that {\tt RICE} consistently outperforms the self-imitation learning method. While self-imitation learning has the advantage of encouraging the agent to imitate past successful experiences by prioritizing them in the replay buffer, it cannot address scenarios where the agent (and its past experience) has errors or sub-optimal actions. In contrast, {\tt RICE} constructs a mixed initial distribution based on the identified critical states (using explanation methods) and encourages the agent to explore the new initial states. This helps the agent escape from local minima and break through the training bottlenecks.

\begin{table}[ht]
    \centering    
    \small
    \caption{Performance comparison between Self-Imitation Learning (SIL) and {\tt RICE} on four MuJoCo tasks.}
    \begin{tabular}{c?c|c|c|c}
        \Xhline{1pt}
        Method & Hopper & Walker2d & Reacher & HalfCheetah \\
        \Xhline{1pt}
        SIL & 3646.46 (23.12) & 3967.66 (1.53) & -2.87 (0.09) & 2069.80 (3.44) \\\hline
        Ours & \textbf{3663.91 (20.98)} & \textbf{3982.79 (3.15)} & \textbf{-2.66 (0.03)} & \textbf{2138.89 (3.22)} \\
        \Xhline{1pt}
    \end{tabular}
    \label{tab:self_imitation}
\end{table}

\para{Impact of Other Explanation Methods.}
We investigate the impact of other explanation methods (\ie AIRS~\cite{yu2023airs} and Integrated Gradients~\cite{sundararajan2017axiomatic}) on four Mujoco games. we fix the refining method and use different explanation methods to identify critical steps for refinement. The results are reported in~\autoref{tab:other_explanation}. We observe that using the explanation generated by our mask network, the refining achieves the best outcome across all four applications. Using other explanation methods (Integrated Gradients and AIRS), our framework still achieves better results than the random baseline, suggesting that our framework can work with different explanation method choices.

\begin{table}[t]
    \centering
    \small
    \caption{Performance comparison when using different explanation methods across four MuJoCo tasks.}
    \begin{tabular}{c?c|c|c|c}
        \Xhline{1pt}
        Task & Random Explanation & Integrated Gradients & AIRS & Ours \\
        \Xhline{1pt}
        Hopper & 3648.98 (39.06) & 3653.24 (14.23) & 3654.49 (8.12) & \textbf{3663.91 (20.98)} \\\hline
        Walker2d & 3969.64 (6.38) & 3972.15 (4.77) & 3976.35 (2.40) & \textbf{3982.79 (3.15)} \\\hline
        Reacher & -3.11 (0.42) & -2.99 (0.31) & -2.89 (0.19) & \textbf{-2.66 (0.03)} \\\hline
        HalfCheetah & 2132.01 (0.76) & 2132.81 (0.83) & 2133.98 (2.52) & \textbf{2138.89 (3.22)} \\
        \Xhline{1pt}
    \end{tabular} 
    \label{tab:other_explanation}
\end{table}

\para{Sensitivty of $p$ and $\lambda$ in Hopper game with an imitated PPO agent.} We report the sensitivity of hyper-parameters $p$ and $\lambda$ in Hopper game with an imitated PPO agent in \autoref{fig:sensitivity_lambda_ppohopper}. We observe that in general, a mixture probability of $p=0.25$ or $p=0.5$ is a better choice. An RND bonus can facilitate the agent with faster refinement.


\begin{figure*}[htbp]
\centering
\includegraphics[width=\textwidth]{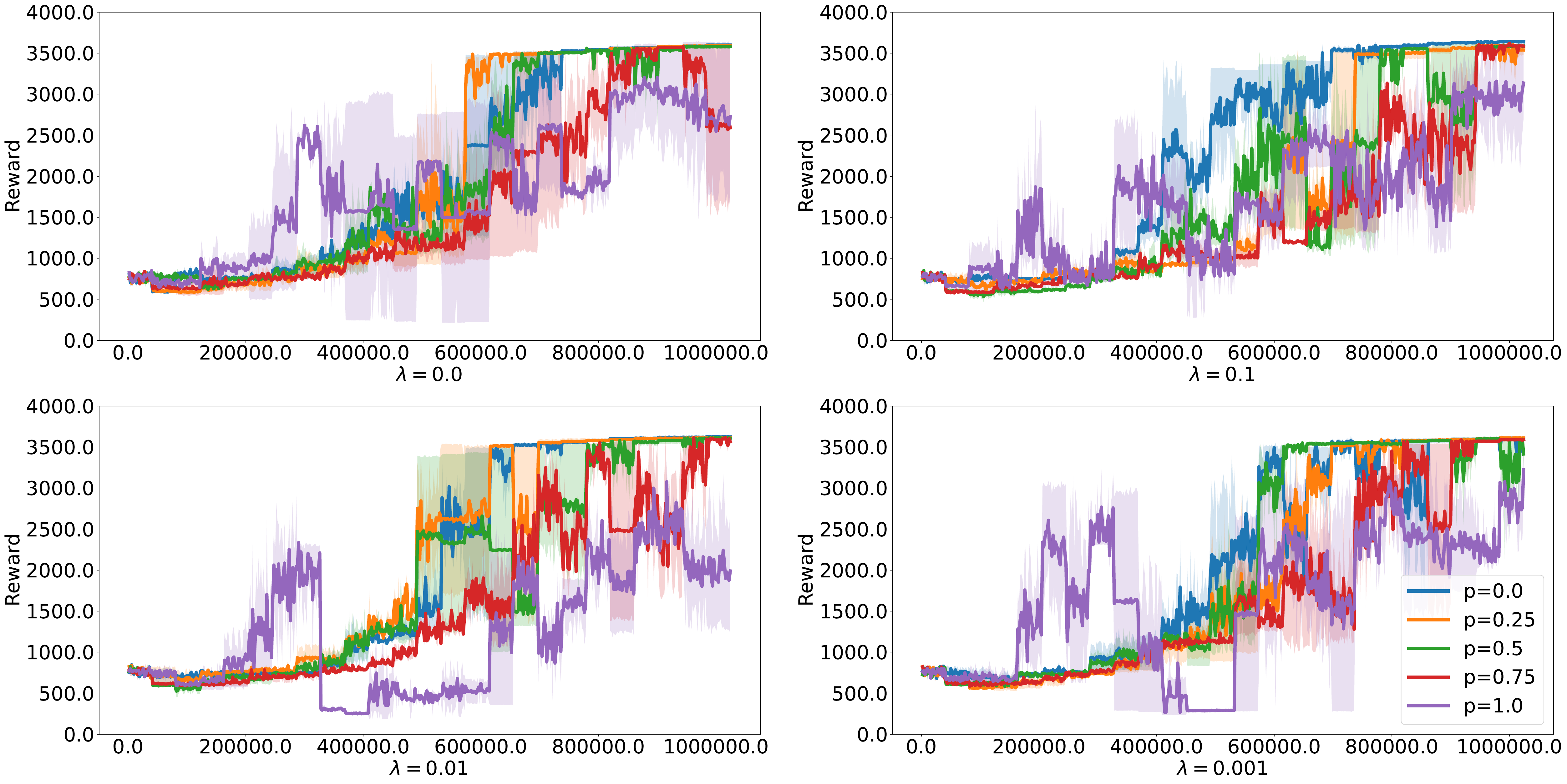}
\caption{Sensitivity results of hyper-parameters $p$ and $\lambda$ in Hopper game with an imitated PPO agent.  We vary the hyper-parameter $\lambda$ from $\{0, 0.1, 0.01, 0.001\}$ and record the performance of the agent after refining. A smaller choice of $\lambda$ means a smaller reward bonus for exploration.}
\label{fig:sensitivity_lambda_ppohopper}
\end{figure*}

\para{Sensitivity of Hyper-parameters $p$ and $\lambda$.} We provide the sensitivity results of $p$ in all applications in \autoref{fig:sensitivity_p_2}. We observe that generally a mixture probability of $p=0.25$ or $p=0.5$ is a good choice. Additionally, 
recall that we need to use the hyper-parameter $\lambda$ to balance the scale of the ``true'' environment reward and the exploration bonus. We test the sensitivity of $\lambda$ from the space \{0.1, 0.01, 0.001\}. \autoref{fig:sensitivity_lambda} reports the agent's performance after refining under different settings of $\lambda$. We observe that our retaining method is insensitive to the choice of $\lambda$. The agent's performance does not vary a lot with different settings of $\lambda$. But $\lambda=0.01$ gives the best performance in all applications except selfish mining.

\begin{figure*}[htbp]
\centering
\includegraphics[width=\textwidth]{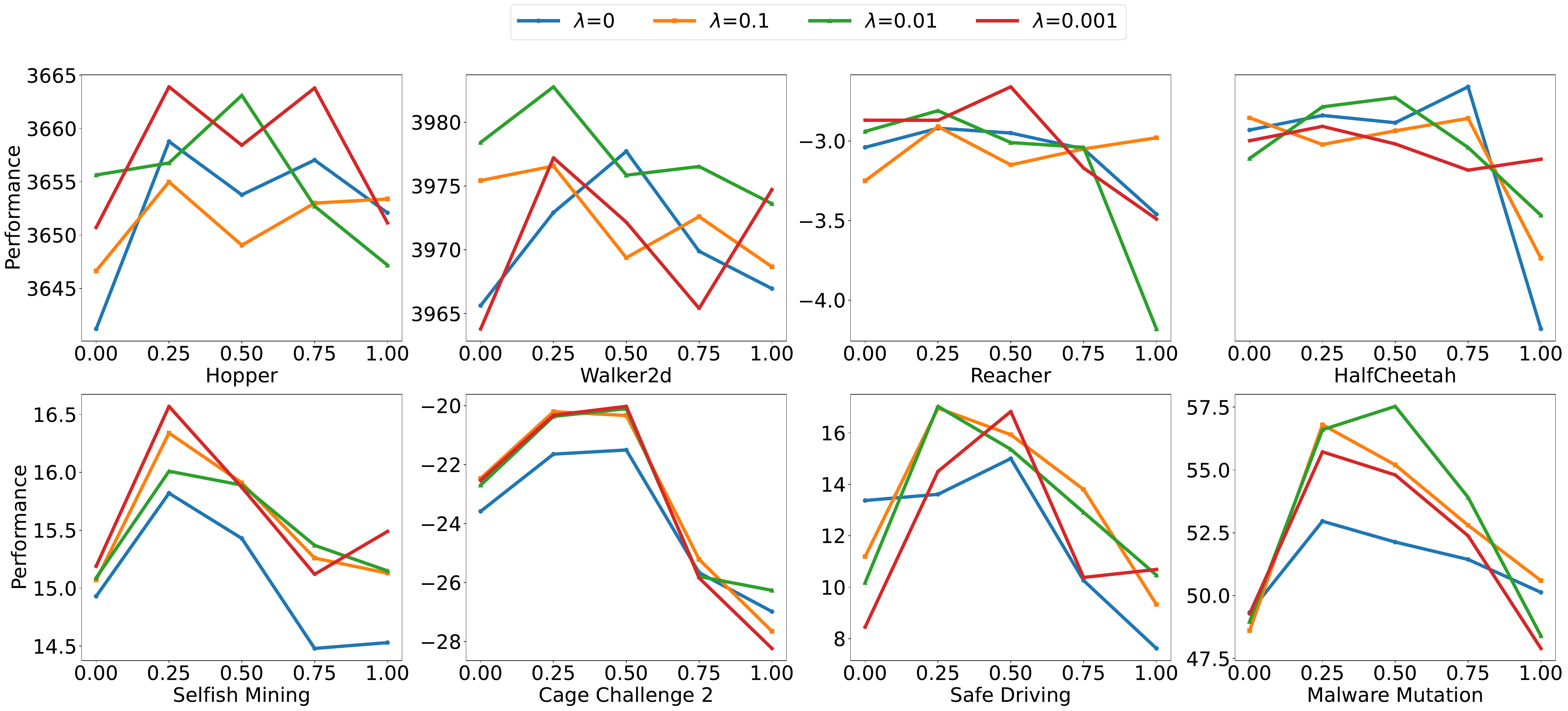}
\caption{Sensitivity results of hyper-parameter $p$ in all applications. We vary the hyper-parameter $p$ from \{0, 0.25, 0.5, 0.75, 1\} under different $\lambda$, and record the performance of the agent after refining. When $p=0$, refining starts from the default initial states of the environment. When $p=1$, refining starts exclusively from critical states. We show that the ``mixed'' initial state distribution helps to achieve a better performance.}
\label{fig:sensitivity_p_2}
\end{figure*}

\begin{figure*}[htbp]
\centering
\includegraphics[width=\textwidth]{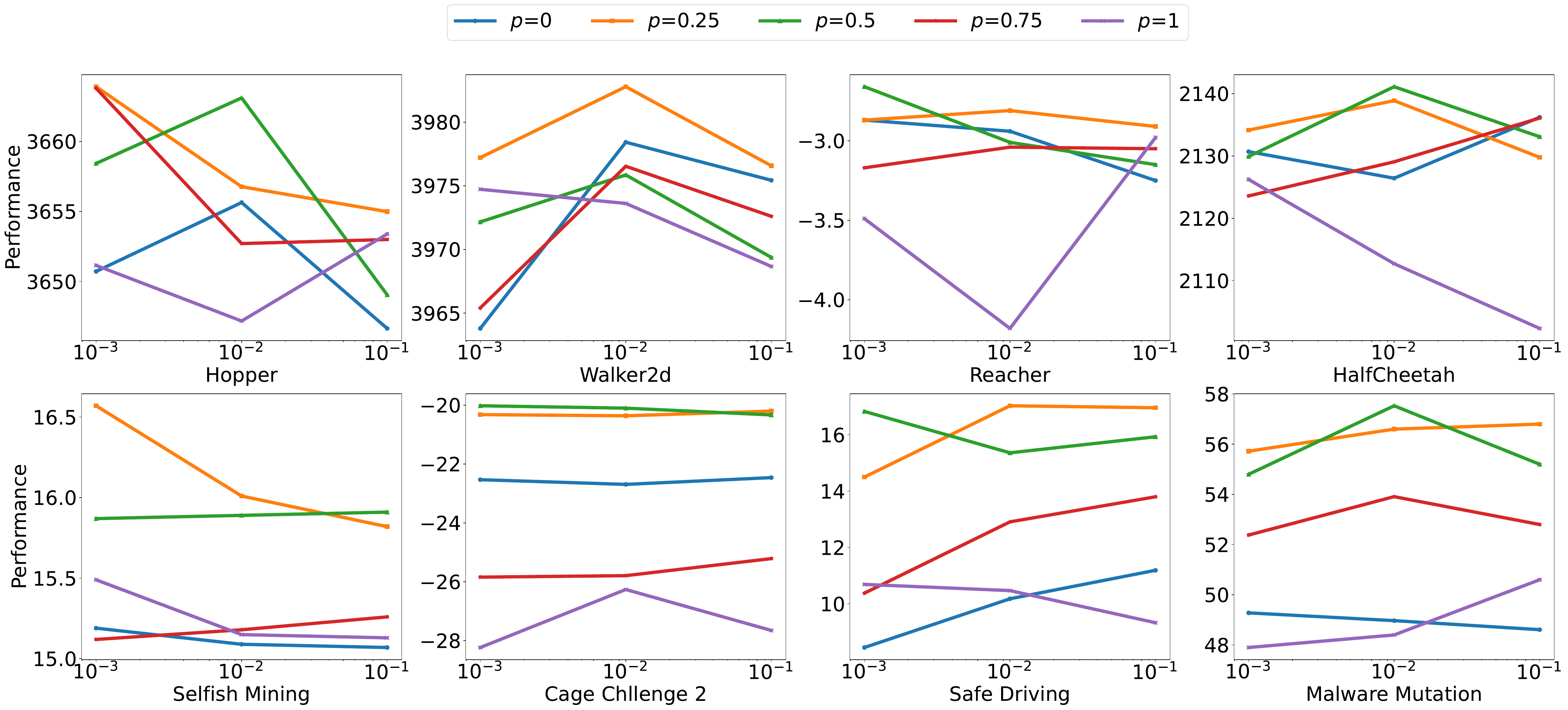}
\caption{Sensitivity results of hyper-parameter $\lambda$.  We vary the hyper-parameter $\lambda$ from $\{0.1, 0.01, 0.001\}$ and record the performance of the agent after refining. A smaller choice of $\lambda$ means a smaller reward bonus for exploration.}
\label{fig:sensitivity_lambda}
\end{figure*}

\para{Sensitivity of $\alpha$.} Recall that under certain assumptions, we are able to simplify the design of StateMask. We propose an intrinsic reward mechanism to encourage the mask network to blind more states without sacrificing performance. The hyper-parameter $\alpha$ is then introduced to balance the performance of the perturbed agent and the need for encouraging blinding. We test the sensitivity of $\alpha$ from the space $\{0.01, 0.001, 0.0001\}$ and report the fidelity scores under different settings of $\alpha$ in \autoref{fig:sensitivity_alpha}. We observe that though the value of $\alpha$ varies, the fidelity score does not change much. 

\begin{figure}[htp]
\centering
\includegraphics[width=0.9\textwidth]{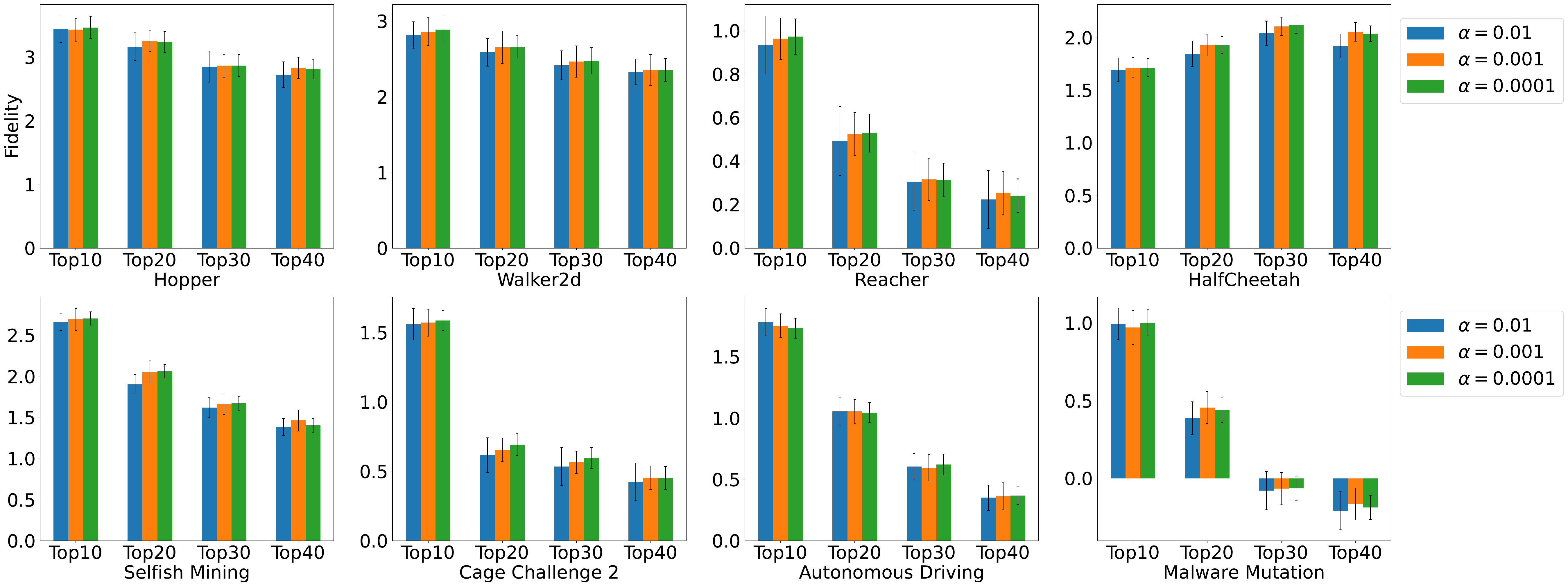}
\caption{Sensitivity results of hyper-parameter $\alpha$.  We vary the hyper-parameter $\alpha$ from $\{0.01, 0.001, 0.0001\}$ and record the fidelity scores of the mask network trained under different settings of $\alpha$. A higher fidelity score means a higher fidelity.}
\label{fig:sensitivity_alpha}
\end{figure}

\subsection{Evaluation Results of MuJoCo Games with Sparse Rewards}
\label{appendix:sparse_games}

\para{Results of SparseWalker2d.} First, we compare our refining method with other baseline methods (\ie PPO fine-tuning, StateMask-R, and JSRL) in the SparseWalker2d game. \sref{Figure}{fig:sparse_halfcheetah} shows that our refining method is able to help the DRL agent break through the bottleneck with the highest efficiency compared with other baseline refining methods. Additionally, by replacing our explanation method with a random explanation, we observe that the refining performance is getting worse.

\begin{figure*}[t]
\centering
\includegraphics[width=\textwidth]{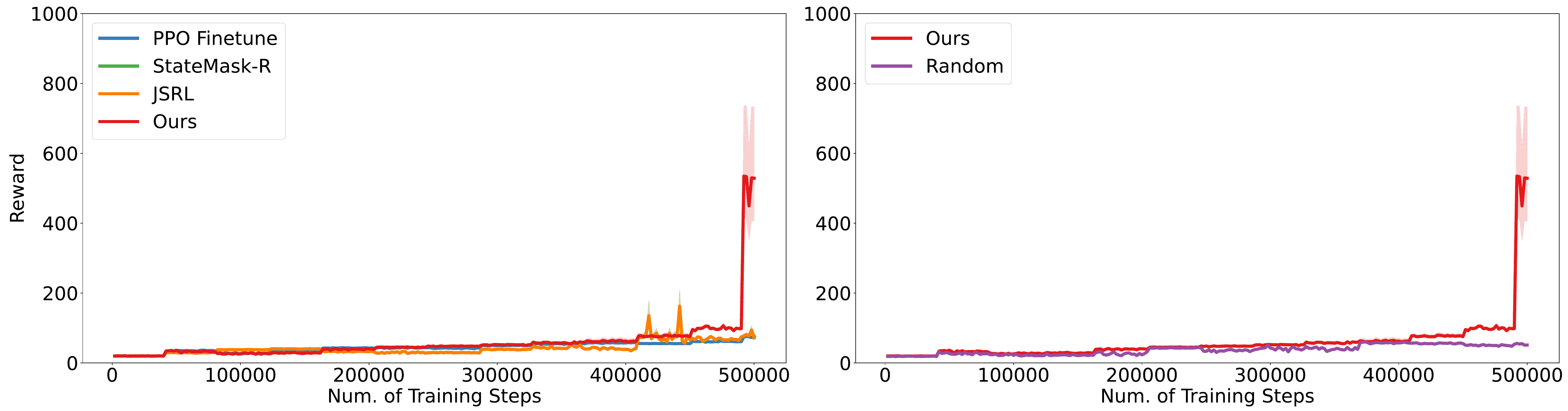}
\caption{Agent Refining Performance in the SparseWalker2d Games.
For the left figure, we fix the explanation method to our method (mask network) if needed while varying refining methods. For the right figure, we fix the refining method to our method while varying the explanation methods. }
\label{fig:sparse_halfcheetah}
\end{figure*}

\para{Sensitivity of $p$ and $\lambda$.} We report the sensitivity of hyper-parameters $p$ and $\lambda$ in the three MuJoCo games with sparse rewards in~\autoref{fig:sensitivity_lambda_sparse_hopper},~\autoref{fig:sensitivity_lambda_sparse_walker}, and~\autoref{fig:sensitivity_lambda_sparse_halfcheetah}. We have the following observations: First, generally, a mixed probability $p$ within the range of 0.25 and 0.5 would be a good choice. Second, the refining benefits from the exploration bonus in the sparse MuJoCo games. Third, PPO fine-tuning cannot guarantee that the refined agent can achieve a good performance. Especially in SparseWalker2d game, we observe that ppo fine-tuning cannot break through the training bottleneck of the DRL agent.


\begin{figure*}[htbp]
\centering
\includegraphics[width=\textwidth]{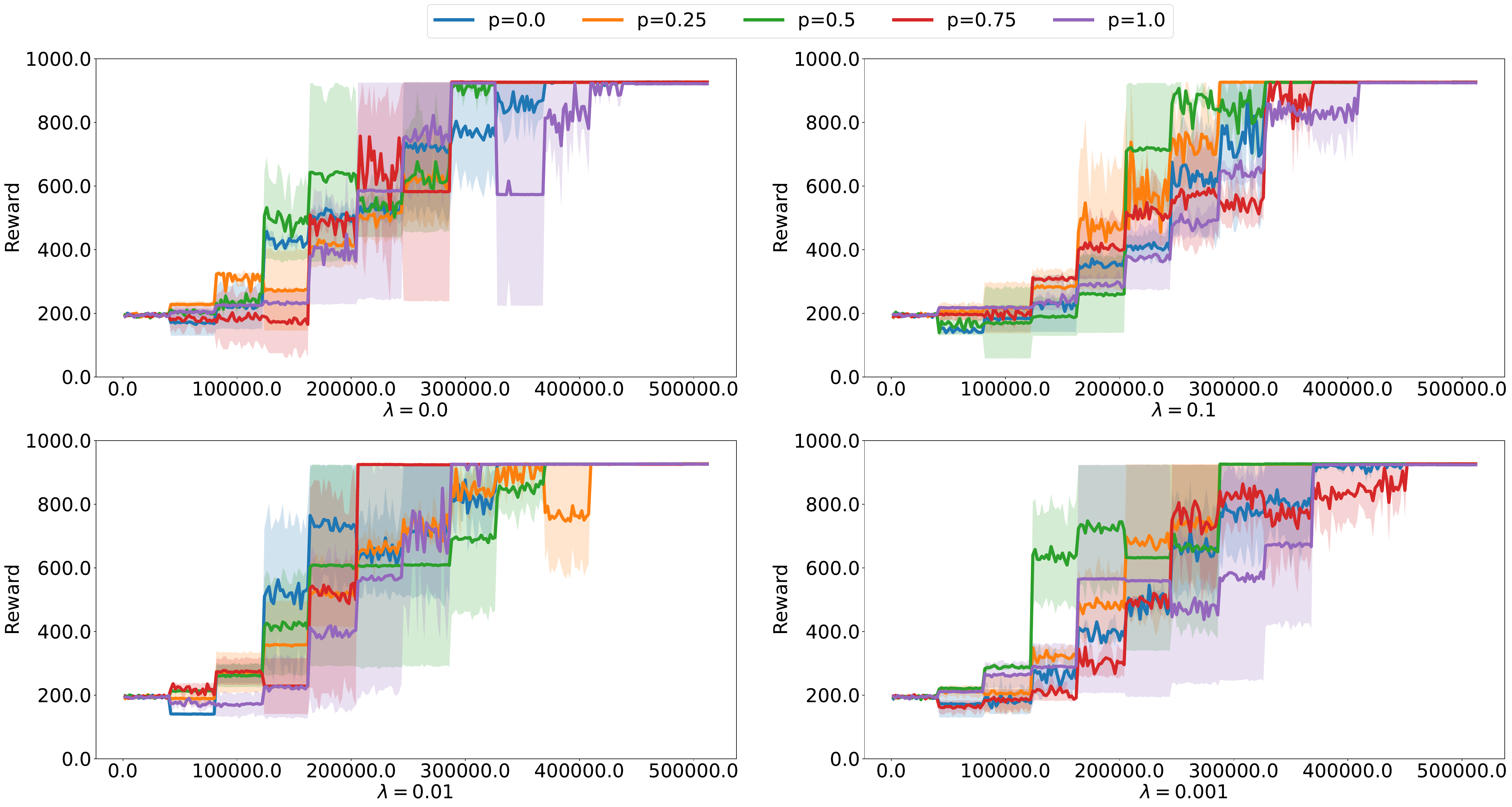}
\caption{Sensitivity results of hyper-parameter $\lambda$ in SparseHopper game.  We vary the hyper-parameter $\lambda$ from $\{0, 0.1, 0.01, 0.001\}$ and record the performance of the agent after refining. A smaller choice of $\lambda$ means a smaller reward bonus for exploration.}
\label{fig:sensitivity_lambda_sparse_hopper}
\end{figure*}


\begin{figure*}[htbp]
\centering
\includegraphics[width=\textwidth]{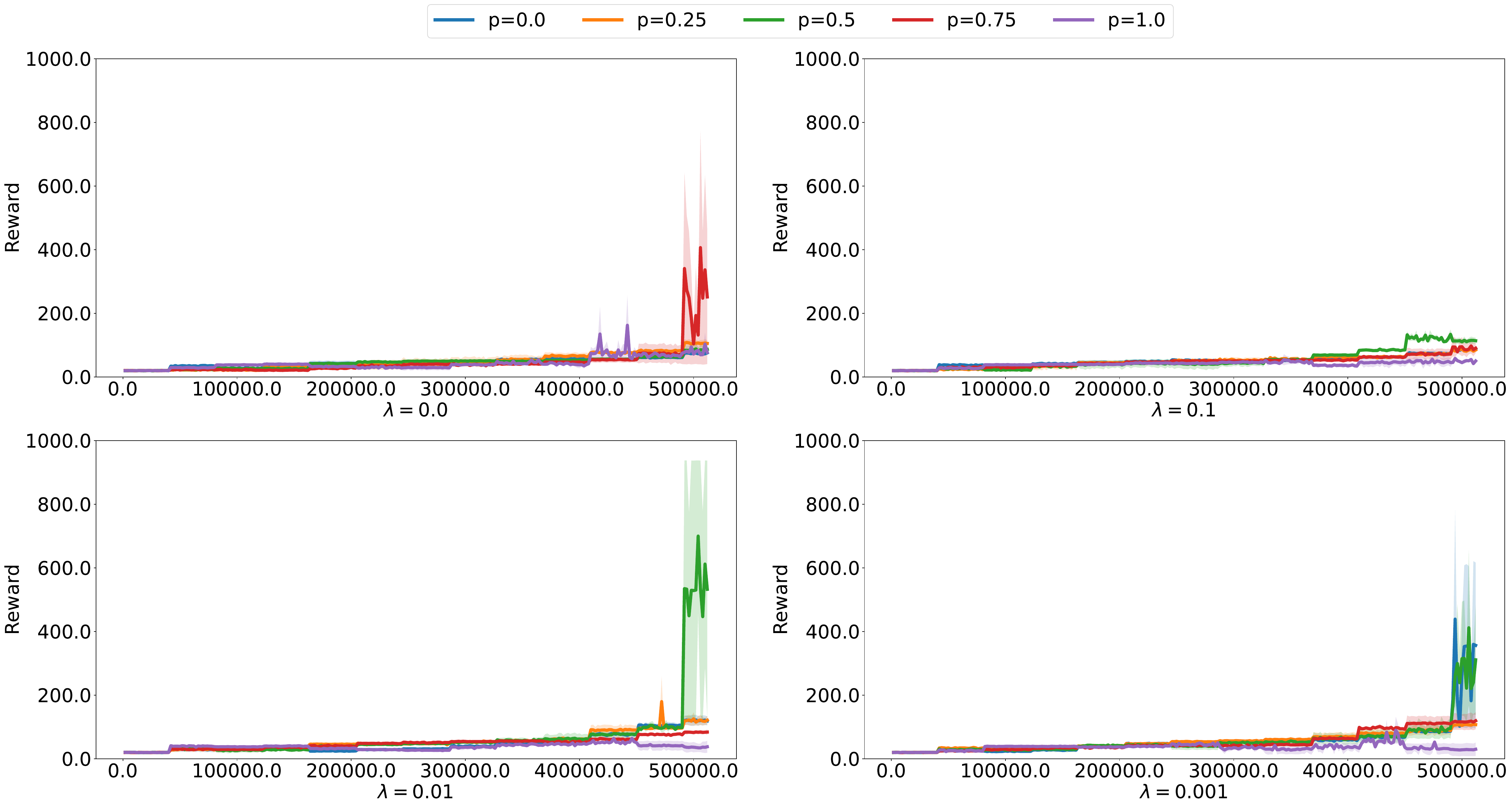}
\caption{Sensitivity results of hyper-parameter $\lambda$ in SparseWalker2d game.  We vary the hyper-parameter $\lambda$ from $\{0, 0.1, 0.01, 0.001\}$ and record the performance of the agent after refining. A smaller choice of $\lambda$ means a smaller reward bonus for exploration.}
\label{fig:sensitivity_lambda_sparse_walker}
\end{figure*}

\begin{figure*}[htbp]
\centering
\includegraphics[width=\textwidth]{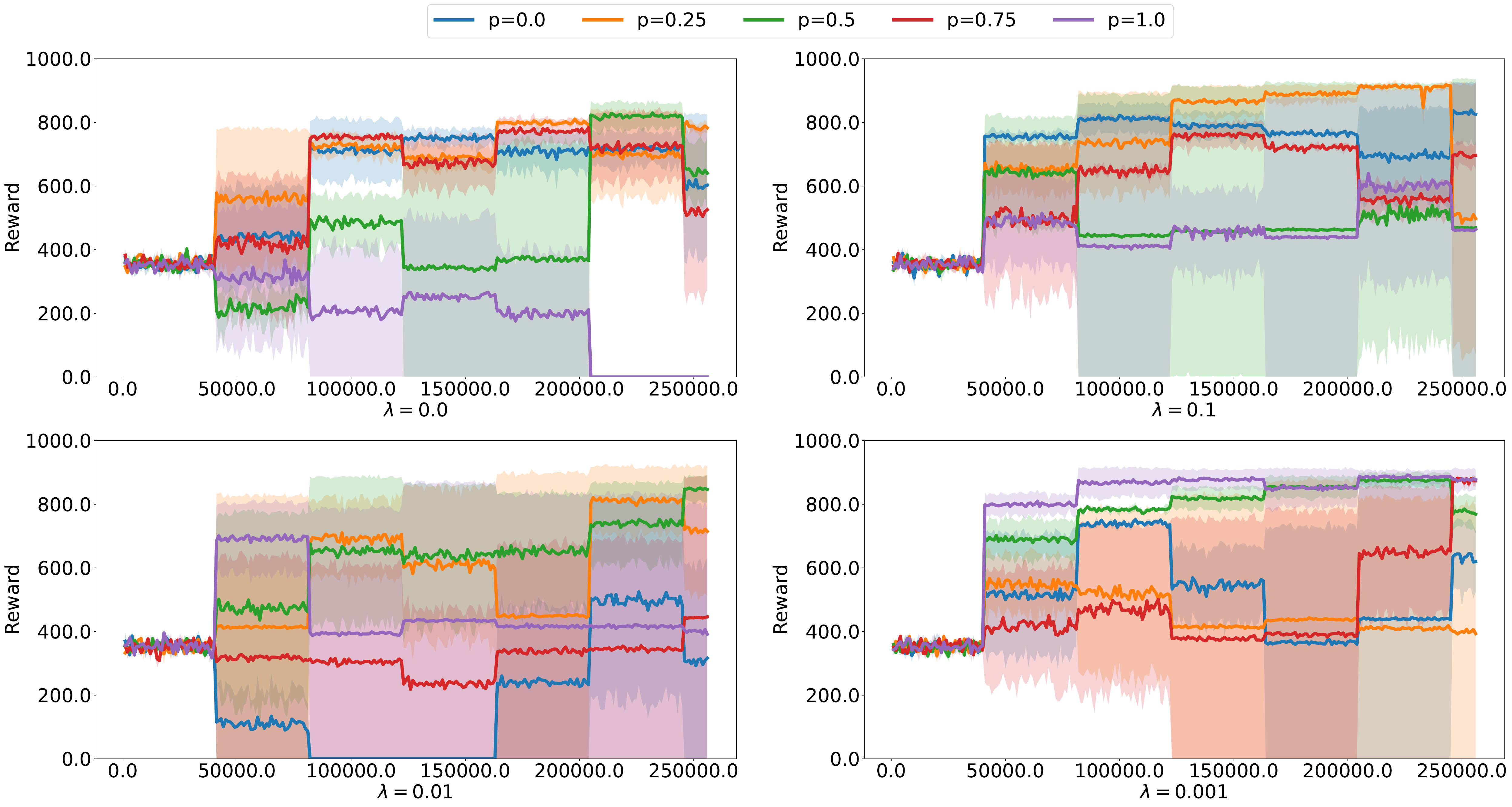}
\caption{Sensitivity results of hyper-parameter $\lambda$ in SparseHalfCheetah game.  We vary the hyper-parameter $\lambda$ from $\{0, 0.1, 0.01, 0.001\}$ and record the performance of the agent after refining. A smaller choice of $\lambda$ means a smaller reward bonus for exploration.}
\label{fig:sensitivity_lambda_sparse_halfcheetah}
\end{figure*}

\subsection{Qualitative Analysis}
\label{appendix:qualitative}

\begin{figure}[t]
    \centering
    \includegraphics[width=\textwidth]{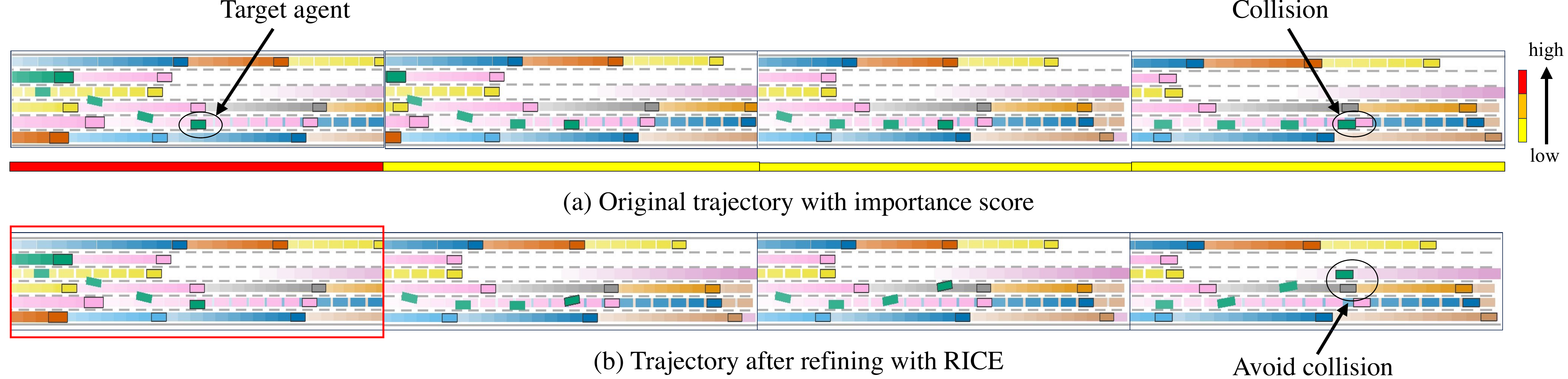}
    \caption{(a): In the original trajectory, the target agent (the green car) eventually collides with the pink car, which is an undesired outcome. Each time step is marked with a different color: ``yellow'' indicates the least important step and “red” represents the most important step. (b): We highlight the critical states identified by our explanation method and the corresponding outcome after refining. Using our explanation method, the target agent (the green car) successfully avoids collision.}
    \label{fig:qualitative}
\end{figure}

We do a qualitative analysis of the autonomous driving case to understand how {\tt RICE} impacts agent behavior and performance. We visualize the agent's behavior before and after refining the agent. \autoref{fig:qualitative}(a) shows a trajectory wherein the target agent (depicted by the green car) fails to reach its destination due to a collision with a pink car on the road. Given the undesired outcome, we use our method to identify the critical steps that contribute to the final (undesired) outcome. The important steps are highlighted in red color. Our method identifies the important step as the one when the green car switches across two lanes into the lane of the pink car. The critical state is reasonable because this early step allows the green car to switch lanes to avoid the collision. Based on the provided explanation, we apply our refining method to improve the target agent. The trajectory after refining is shown in \autoref{fig:qualitative}(b). It shows that after refining, the refined agent (the green car) successfully identifies an alternative path to reach the destination while avoiding collision.

\section{Case Study: Malware Mutation}
\label{appendix:malware}

\subsection{Design Intuitions}
\begin{table*}[t]
\centering
\caption{
{\bf Malware Mutation Case Study}---%
{\small
We evaluate the evasion probability of the agent under different settings and count the corresponding action frequencies.}
}
\resizebox{\textwidth}{!}{ 
\begin{tabular}{l|l|l|c}
\Xhline{1pt}
{\bf Refine Setting}                                                & {\bf Test Setting}                      & {\bf Action Frequency}                             & {\bf Evasion} \\ \Xhline{1pt}
Original agent w/o refinement                                      & From default initial S & \{$A_4$: 4,914, $A_9$: 5\}                        & 33.8\%              \\ \hline
Continue training                                              & From default initial S & \{$A_4$: 2,590, $A_7$: 55, $A_1$: 99, $A_9$: 95\}     & 38.8\%              \\ \hline
\multirow{2}{*}{Refine from critical states}               & From critical states        & \{$A_{12}$: 2,546, $A_5$: 138, $A_4$: 32, $A_9$: 8\}    & 50.8\%              \\ \cline{2-4} 
                                                               & From default initial S & \{$A_{12}$: 4,728, $A_5$: 62\}                      & 36.2\%              \\ \hline
Refine from mixed initial state dist.               & From default initial S & \{$A_4$: 1,563, $A_{12}$: 1,135, $A_5$: 332, $A_6$: 12\} & 58.4\%              \\ \hline
Refine from mixed initial state dist. + exploration & From default initial S & \{$A_5$: 2,448, $A_7$: 165, $A_{12}$: 138, $A_4$: 6\}   & 68.2\%              \\ \Xhline{1pt}
\end{tabular}
}
\label{table:case_malware}
\end{table*}

First, we use malware mutation as a case study to confirm our design intuitions before the proposed refining method.
Recall that the refining method contains three important ideas. 
First, we integrate the explanation result (identified critical step) into the refining process. 
Second, we design a mixed initial state distribution to guide the refining of the target agent. 
Third, we encourage the agent to perform exploration for diverse states during the refining phase. 
In the following, we create multiple baselines by gradually adding these ideas to a naive baseline to show the contribution of each idea. 
We also provide evidence to support our stance against overfitting. 
\autoref{table:case_malware} summarizes the results. 

To start, the original agent is trained for 100 epochs until convergence. We test the target agent for 500 runs, resulting in an average evasion probability of 33.8\%. To extract behavioral patterns, we perform a frequency analysis on the mutation actions taken by the agent across all 500 runs. As shown in the first row of \autoref{table:case_malware}, there is a clear preference for $A_4$ (\ie ``add\_bytes\_to\_section\_cave''). A complete list of the possible actions (16 in total) is shown in \autoref{table:malconv_action} (Appendix).



\para{Continue Learning w/o Explanation.}
The most common refining method is to lower the learning rate and continue training. We continue to train this target agent using the PPO algorithm for an additional 30 epochs and evaluate its performance over 500 runs. This yields an average evasion probability of 38.8\% (second row in \autoref{table:case_malware}).  It is worth noting that $A_4$ (\ie ``add\_bytes\_to\_section\_cave'') remains the most frequently selected action.

\para{Leverage Explanation Results for Refining.}
Subsequently, we assess the refining outcome by incorporating our explanation result into the refining process. Specifically, we initiate the refining exclusively {\em from the critical steps} identified by the explanation method. For this setting, we do not perform exploration.  

During the test phase, we explore two testing settings. First, we artificially reset the test environment to start from these critical steps. We find that evasion probability surges to 50.8\%. $A_{12}$ (\ie ``modify\_timestamp'') becomes as the most frequently chosen action. This indicates the refined agent learns a policy when encountering the critical state again. However, for more realistic testing, we need to set the test environment to the default initial state (\ie the correct testing condition). Under this setting, we find the evasion probability diminishes to 36.2\%. This stark contrast in results shows evidence of overfitting. The refined agent excels at solving the problem when starting from critical steps but falters when encountering the task from default initial states.



\para{Impact of Mixed Initial State Distribution.} 
Given the above result, we further build a baseline by refining from the proposed mixed initial state distribution (\ie blending the default initial state distribution with the critical states). For this setting, we also do not perform exploration. Through 500 runs of testing, we observe a notable improvement, with the average evasion probability reaching 58.4\% (from the previous baseline's 36.2\%). Furthermore, the action frequency pattern has also undergone a shift. It combines the preferred actions from the two previous refining strategies, highlighting the frequent selection of both $A_4$ and $A_{12}$.

\para{Impact of Exploration.} Finally, we explore the impact of exploration. This baseline represents the complete version of our proposed system by adding the exploration step and using the mixed initial distribution. We notice that the average evasion probability across 500 runs has a major increase, reaching 68.2\%. The most frequent action now is $A_5$ (\ie ``add\_section\_strings''). $A_4$ and $A_{12}$ are still among the top actions but their frequencies are lowered. This shows the benefits of exploring previously unseen states and diverse mutation paths. In return, the refined agent is able to get out of the local minima to identify more optimal policies.

\subsection{Discovery of Design Flaws}

Additionally, our explanation results have led to the discovery of design flaws in the malware mutation application~\cite{raff2017malware}. We will further explain how we use {\tt RICE} to identify these problems. 


\para{Questions and Intuitions.} 
When using {\tt RICE} to explain the malware mutation process, we observe a scenario where the agent constantly chooses the same action ``upx\_pack'' in multiple consecutive steps. According to the agent, these actions receive a similar reward. However, {\tt RICE} (our mask network) returns different ``explanations'' for these steps (\ie they have highly different importance scores). According to {\tt RICE}, only the first action holds a high importance score, while the other consecutive actions barely have an impact on the final reward (\ie they appear redundant). This raises the question: why does the agent assign a similar reward to these consecutive steps in the first place?  

Another interesting observation is from refining experiments. We find that PPO-based refining cannot yield substantial improvements. While we have expected that these methods do not perform as well as ours (given our exploration step), the difference is still bigger than we initially expected. This motivates us to further examine the reward function design to explore whether it has inadvertently discouraged the DRL agent from finding good evasion paths. 



\para{Problems of Reward Design.} Driven by the intuitions above, we examined the reward design and identified two problems.  Firstly, the reward mechanism is inherently {\em non-Markovian} which deviates from the expectation of a typical reinforcement learning (RL) framework. In typical RL settings, rewards are contingent on the current state $s$ and the next state $s'$. However, the current design computes the reward based on the {\em initial state} $s_0$ and the subsequent state $s'$. Consequently, this may assign an identical reward for the same action (\eg ``upx\_pack'') in consecutive steps. This non-Markovian nature of the reward mechanism can mislead the DRL agent and hurt its performance. 

Second, we find that the {\em intermediate rewards} exhibit unusually high sparsity, \ie many intermediate rewards tend to have a value close to zero, which poses a significant challenge for the PPO algorithm to learn a good policy based on such intermediate rewards. Agents refined with these methods can be easily trapped in local minima. 


\para{Fixing the Problematic Reward Design.} Based on these insights, we fix the bugs in the reward design with two simple steps: (1) We make the reward function Markovian, which depends only on the current state and the next state. (2) We perform scaling on the intermediate reward with a coefficient of 3. After that, we re-run an experiment to evaluate the correctness of our modifications. We train a DRL agent for 100 epochs with the same parameters under the new reward design and test its performance over 3 trials of 500 runs. The experiment shows that the evasion probability of the agent under the new reward design jumps from 42.2\% (using the old reward function, see \autoref{table:retrain_method}) to 72.0\%, which further confirms our intuitions. This case study illustrates how developers can use {\tt RICE} to debug their system and improve their designs.

\section{Limitation}
\label{appendix:limitation}
We use the continuous ``Mountain Car'' environment~\cite{mountaincar} as a negative control task to illustrate a scenario where {\tt RICE} does not work well. In this “extreme” case, \sref{Assumption}{assumption:mismatch} is completely broken since the state coverage of the pre-trained agent is limited to a small range around the initial point. In this experiment, we train a target agent using Proximal Policy Optimization (PPO) for 1 million steps. The results show that the policy performance remained poor, with the agent frequently getting trapped at the starting point of the environment. In such cases where the original policy fails to learn an effective strategy, the role of explanations becomes highly limited. Since {\tt RICE} relies on the identified critical states to enhance the policy, if the policy itself is extremely weak (\ie not satisfying \sref{Assumption}{assumption:mismatch}), then the explanations will not be meaningful, which further huts the refinement. In the case of the Mountain Car experiment, {\tt RICE} essentially reduces to being equivalent to Random Network Distillation (RND) due to the lack of meaningful explanation. We show the result when refining the pre-trained agent using our method and RND in \autoref{fig:mountain_car}.

\begin{figure*}[t]
\centering
\includegraphics[width=0.6\textwidth]{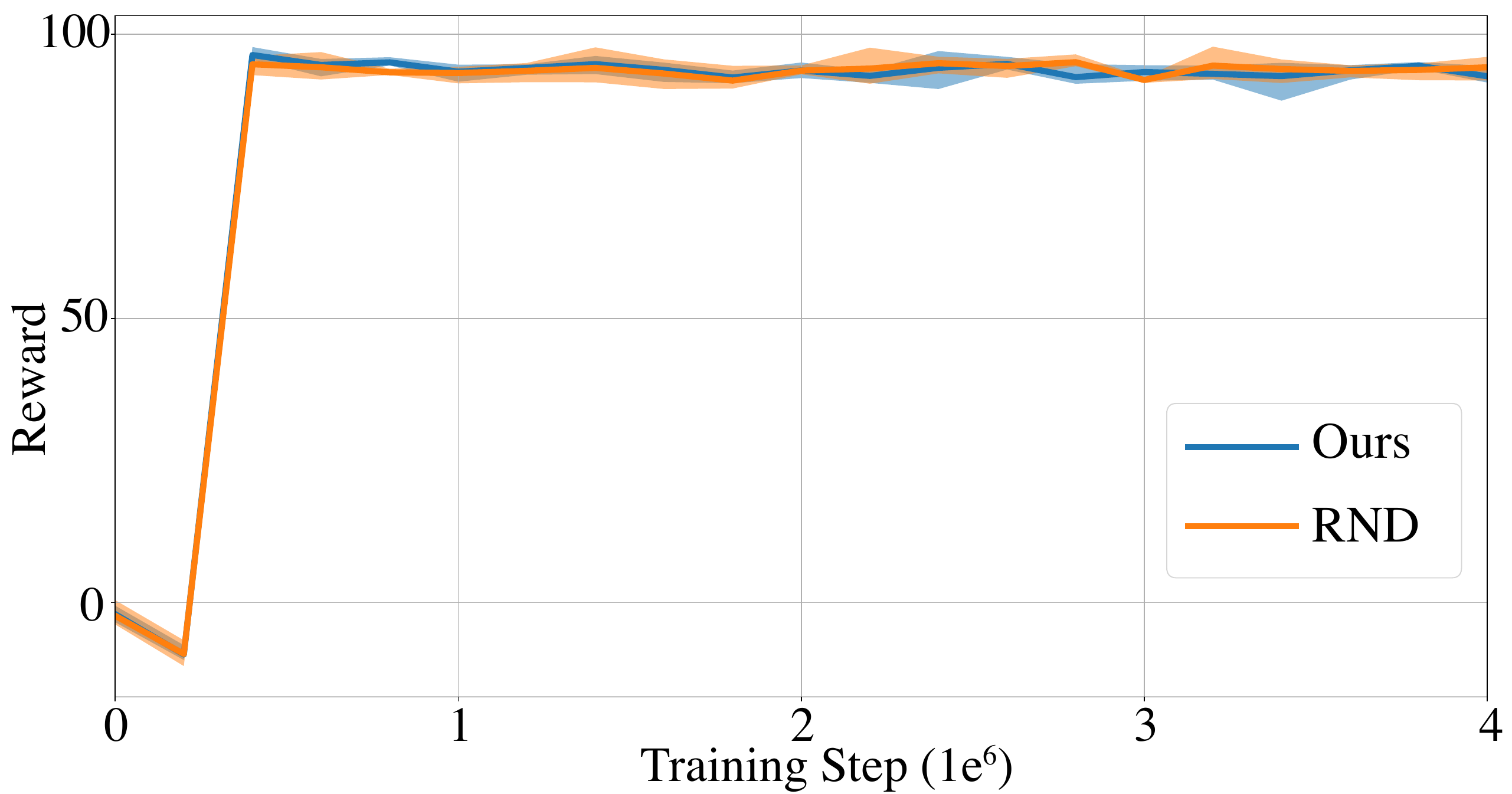}
\caption{Refining performance with our method and RND method in MountainCarContinuous-v0 game. The state coverage of the pre-trained policy is limited to a small range around the initial point.}
\label{fig:mountain_car}
\end{figure*}

\end{document}